\documentclass[runningheads]{llncs}
\usepackage[T1]{fontenc}
\usepackage{graphicx}
\usepackage{booktabs}
\usepackage[misc]{ifsym}
\newcommand{\corr}{(\Letter)}
\usepackage{amsfonts, amssymb, amsmath, mathrsfs}
\usepackage[ruled, linesnumbered]{algorithm2e}
\usepackage{xpatch}
\usepackage{multirow}
\usepackage{subfig}
\usepackage{color}
\usepackage{url}
\usepackage{hyperref}

\makeatletter
\newcommand{\printfnsymbol}[1]{%
  \textsuperscript{\@fnsymbol{#1}}%
}
\makeatother

\begin{document}

\title{A Mathematics Framework of Artificial Shifted Population Risk and Its Further Understanding Related to Consistency Regularization}

\titlerunning{A Math Framework of ASPR and Its Relation to Consistency Regularization}

\author{Xiliang Yang \thanks{Equal Contribution} \inst{1} \and
Shenyang Deng\printfnsymbol{1}\inst{1}  \and
Shicong Liu\printfnsymbol{1} \inst{1}\ \and Yuanchi Suo\printfnsymbol{1} \inst{1}\\and Wing.W.Y  NG \inst{1} \and Jianjun  Zhang \inst{2} \corr}

\authorrunning{X.Yang et al.}

\institute{South China University of Technology, GuangZhou GuangDong  510641, China \email{\{xlyangscut,shenyangdeng2023\}@gmail.com,\{shicong\_liu ,aksldhfjg\}@163.com,wingng@ieee.org}
\and
South China Agricultural University, GuangZhou GuangDong  510642, China \email{jzhangcs@gmail.com}
}

\toctitle{A Mathematics Framework of Artificial Shifted Population Risk and Its Further Understanding Related to Consistency Regularization}
\tocauthor{Xiliang~Yang,Shenyang~Deng,Shicong~Liu,Yuanchi~Suo,Wing.W.Y~NG,Jianjun~Zhang}

\maketitle              

\begin{abstract}
Data augmentation is an important technique in training deep neural networks as it enhances their ability to generalize and remain robust. While data augmentation is commonly used to expand the sample size and act as a consistency regularization term, there is a lack of research on the relationship between them. To address this gap, this paper introduces a more comprehensive mathematical framework for data augmentation. Through this framework, we establish that the expected risk of the shifted population is the sum of the original population risk and a gap term, which can be interpreted as a consistency regularization term. The paper also provides a theoretical understanding of this gap, highlighting its negative effects on the early stages of training. We also propose a method to mitigate these effects. To validate our approach, we conducted experiments using same data augmentation techniques and computing resources under several scenarios, including standard training, out-of-distribution, and imbalanced classification. The results demonstrate that our methods surpass compared methods under all scenarios in terms of generalization ability and convergence stability. We provide our code implementation at the following link: \url{https://github.com/ydlsfhll/ASPR}.

\end{abstract}

\section{Introduction}
Data augmentation creates a training dataset using synthetic data from the prior knowledge. It improves the generalization of machine learning models, particularly in the case of deep neural networks. For decades, its reliable performance has been verified in various of computer vision tasks such as image classification \cite{01krizhevsky2017imagenet,02szegedy2015going,03he2016deep} and object detection \cite{04ren2015faster,05lin2017feature}. 
To the best of our knowledge, there are currently two major explanations for the role of this technique. The first one views data augmentation as simply increasing the sample size, and explains it with statistical tools such as VC dimension theory \cite{06vapnik1998statistical}. The other one \cite{huang2023towards,wang2022toward} views data augmentation as a regularization method, which train the model on a more complex population, which is called shifted population by injecting noise with prior knowledge to the original population, thereby enabling the model to retain semantic information unchanged.

However, the model is ultimately trained with the augmented samples, thereby improving the model's performance on the original population. Therefore, it is important to further explore the relationship between the expected risk of these two populations. To address this issue, we develop a rigorous mathematical framework of the shifted population $p^*(x')$ and data augmentation. Based on this framework, we prove that the expected risk of the shifted population is the summation of the original population and a gap term that can be viewed as a consistency regularization term. This decomposition sheds light on the unification of the two aforementioned explanations. Moreover, inspired by the work of \cite{16he2019data}, the generalization of the model greatly depends on the consistency between the empirical risk of the original population and the shifted one, and the gap term may violate such consistency. To address this issue, we add a trade-off coefficient to the gap term to highlight the importance of the learning of major features, which is controlled by the expected risk of $p(x)$. This approach greatly benefits the performance of the model.

At present, some work like \cite{chen2020group} has provided a decent mathematical framework for data augmentation, but it is too limited to describe some of the existing data augmentations, and it completely ignores the gap term. However, this neglect could be harmful, for it is indicated by our analysis and experiment that reducing its impact in early stages of training has been proven to be helpful for the model's generalization. Please see Appendix \ref{sec:diff_group} for a more detailed discussion.

We conducted experiments to evaluate the proposed training strategy on popular image classification benchmarks, namely CIFAR-10/100 \cite{krizhevsky2009learning}, Food-101 \cite{bossard2014food}, and ImageNet (ILSVRC2012) \cite{bossard2014food}. Our evaluation involved using representative deep networks such as ResNet-18, ResNet-50, and WideResNet-28-10. In addition to assessing the performance in the standard scenario, we also tested the algorithm in the out-of-distribution (OOD) scenario with dataset PACS \cite{li2017deeper}) and the 
long-tail imbalanced classification (LT) scenario with dataset LT-CIFAR10 \cite{cui2019class}. Across all our experiments, our strategy consistently achieved lower error rates and demonstrated more stable convergence compared to the standard data augmentation strategy.

This paper's contributions can be summarized as follows: 
\begin{enumerate}
    \item  We provide a rigorous mathematical definition for the shifted distribution $p^*(x')$ of the augmented samples, which further reveals that the commonly used augmented samples actually comes from the a conditional distribution $p(x'|x)$. We also give a mathematical description of sampling from this distribution and find that the samples used during training from this marginal distribution are not completely independent, which is surprising.
    \item Based on the proposed mathematical framework, we discover that the risk on the shifted population $p^*(x')$ can be decomposed into a risk on the original population $p(x)$ and a gap term, serving as a consistency regularization term. 
    \item We provide a theoretical understanding of such decomposition and an explanation of why our training strategy is beneficial for the improvement of generalization.
\end{enumerate}

\section{Related Work}
\paragraph{Data Augmentation Frame Work}
Data augmentation methods play a crucial role in improving the performance of machine learning models in practical applications. These methods encompass a range of techniques, including traditional fixed augmentation methods like Cutout \cite{07devries2017improved}, Mixup \cite{08zhang2017mixup}, and Cutmix  \cite{09yun2019cutmix}. Additionally, there are adaptive augmentation methods such as AutoAugment \cite{Cubuk_2019_CVPR}, Fast AutoAugment \cite{NEURIPS2019_6add07cf}, DADA \cite{li2020dada}, and CMDA \cite{tian2021continuous}, which dynamically design augmentations based on the dataset. Despite the availability of these diverse augmentation methods, there is a dearth of theoretical frameworks for analyzing the population shift phenomenon induced by data augmentation and the associated shifted population risk.

A recent work \cite{chen2020group} provides a theoretical framework that defines the augmentation operator as a group action. However, their framework has certain limitations, as evidenced by several common augmentation operators that are incompatible with the group action framework, as detailed in the Appendix \ref{sec:diff_group}. Our proposed framework can be applied to a wider range of data augmentation operators compared to theirs.

\paragraph{Population Shift}
Population shift is a common concern in machine learning robustness and generalization problems. It refers to a problem in which the population of data changes during some processes, such as a distribution being transformed to other distributions within the same distribution family, and the change of the parameters of a distribution. A common example for population shift in machine learning is the different semantic styles between the training and testing sets, such as PACS \cite{li2017deeper}, Rotated MNIST, Color MNIST \cite{arjovsky2019invariant}, VLCS, and Office-Home \cite{venkateswara2017deep}. However, not all types of population shifts are natural. Style shifts such as PACS are naturally generated distributions, while population shifts such as Rotated MNIST and Color MNIST are artificially generated. It is obvious that all data augmentations will produce an artificial population shift. This work aims to provide a theoretical framework for artificial population shifts and analyze the relationship between the \textbf{shifted population risk} and the \textbf{original population risk}.

\section{Method}
\subsection{Revisiting Data Augmentation with Empirical Risk}

We conduct research in the case of classification and denote the data space and label space as $\mathcal{X}$ and $\mathcal{Y}$ and a joint distribution $p$ is defined on $\mathcal{X}\times \mathcal{Y}$, with marginal distribution $p(x)$ and conditional distribution $p(x|y)$. We call a sample $x$ drawn from $p(x)$ a "clean sample". We aim to train a model $f:\mathcal{X}\rightarrow \mathcal{Y}$ by minimizing the following risk with a loss function $\mathcal{L}(\cdot,\cdot)$:

\begin{equation}\label{cla_risk}
    R_f(p) = \int \mathcal{L}(f(x),y)\mathrm{d}p(x,y),
\end{equation}

As \eqref{cla_risk} is usually intractable, the empirical risk minimization principle is used, aiming at optimizing an unbiased estimator of \eqref{cla_risk} over a training dataset $\mathcal{D} = \{(x_i,y_i)\}_{i=1}^N$:
\begin{equation}\label{emp_cla_risk}
    \hat{R}_f(p) = \frac{1}{N}\sum_{i=1}^N \mathcal{L}(f(x_i),y_i),
\end{equation}
Following \cite{wang2022toward}, we introduce the following assumption to build a bridge between empirical risk and expected risk:
\begin{claim} \label{generalization bound assumption}
Let $C(f)$ be some complexity metric of $f$, $N$ be the number of data (don't have to be independent), $B(N)$ be the "independence" of the input data. For $\forall \delta > 0$, we assume that the following holds with probability $1-\delta$:
\begin{equation}\label{eq:error_rate_ass}
    R_f(p)-\hat{R}_f(p)  \leq \phi(C(f),B(N),\delta).
\end{equation}
Where $\phi(\cdot)$ is a function of these three terms, and it monotonically increases with respect to the second variable. 
\end{claim}

We refer the readers to \cite{homem2008rates} for more detail about the convergence in the non iid case. It is worth noting that data augmentation produces an augmented sample $x'$, which is a distinct random variable from the clean sample $x$, with a different distribution $p^*(x')$ but the same probability space triplet. This leads to a new population $\tilde{p}(x',y)$ and an expected risk defined on it. Specifically, the empirical risk function is defined as follows:
\begin{equation}\label{emp_aug_risk}
    \hat{R}_f(p^*) = \frac{1}{N}\sum_{i=1}^N \mathcal{L}(f(x_i'),y_i).
\end{equation}

It is important to note that minimizing \eqref{emp_aug_risk} does not necessarily result in the minimization of \eqref{cla_risk} or even \eqref{emp_cla_risk}. Meanwhile, data augmentation is also recognized as a regularization technique that can reduce generalization error without necessarily reducing training error \cite{goodfellow2016deep,zhang2021understanding}. Our proposed decomposition as well as the framework should be helpful when one tries to overcome these struggles.

\subsection{The Augmented Neighborhood}
\label{Section:AugNeighborhoog}
Data augmentation is typically applied directly to a clean sample $x$ to generate an augmented sample $x'$. The augmentation is usually designed to preserve the semantic consistency between $x$ and $x'$, hence it is often referred to be "mild". However, the data augmentation is usually controlled by a set of parameter when it is applied to a fixed clean sample $x$. When the parameters are iterated, a large set of augmented samples are produced, among which there are samples are over augmented and should not be considered "mild". As a result, a series of rigorous mathematical definitions are required, so one may draw a line between "ordinary" data augmentation and a "mild" one. 
\subsubsection{The Augmentation and Limitation}
\label{The augmentation and limitation}
We begin this section with the definition of data augmentation:

\begin{definition} \label{definition: AugmentationDef}
Let $\mathcal{X}$   be  the data space , endow $\mathcal{X}$ with Borel $\sigma-$ algebra $\mathcal{F}$, let the data augmentation $A_i(\cdot,\cdot)$ be a map from $\mathcal{X}\times \Theta(A_i) $ to $\mathcal{X}$ satisfying:

\begin{enumerate}
    \item For every fixed $x$ in $\mathcal{X}$, the map $\theta \mapsto A_i(\theta,x)$, is differentiable and injective. We denote the inverse of this map as $h^{-1}_{A_i,x}$.\\
    \item  For every fixed $\theta$ in $\Theta(A_i)$, the map $A_i(\theta, \cdot)$ is an $\mathcal{F}-$ measurable map.\\    
    \item $\forall x \in \mathcal{X} \ \exists \, e_i \in \Theta(A_i) \ \text{s.t.} \ A(e_i,x) = x $ and such $e_i$ is unique.
\end{enumerate}
where $\Theta(A_i)$ is the parameter space of $A_i(\cdot,\cdot)$.
\end{definition}

The differentiability of some popular data augmentations has been proven in \cite{tian2021continuous}. The injectivity of the data augmentation is always guaranteed given proper parameterization and a carefully chosen parameter space. The measurable assumption is required to ensure that $A(\theta,x)$ is still measurable, which is necessary for the adjoint random variable $x'$. However, the tractability of $h^{-1}_{A_i,x}$ is not always guaranteed, but the good news is that it is not always required in practice. More detailed discussion is provided in Section \ref{section:Sampling}, where we discuss how to sample from the conditional distribution $p(x'|x)$.

Denote the set of data augmentation as $\mathcal{A} = \{A_1,\dots A_m\}$, among which $A_i$ corresponds to a certain type of data augmentation such as rotation, Gaussian blur and so on. Denote $\dim (\Theta(A_i)) = d_i$, where  $\Theta(A_i)$ denotes the parameter space of $A_i$. For example, the parameter space of rotation is usually chosen as $(0,2\pi)$ and the dimension is 1. The distribution of the parameter defined on $\Theta(A_i)$ is denoted as $p_i(\mathbf{\theta})$. Now for a given clean sample $x_0$, we consider all of its augmented sample, which is the image of the mapping $A_i(x_0,\cdot)$, defined on $\Theta(A_i)$: 

\begin{definition}\label{definition: SingleAugNeib}
    For any given clean sample $x_0 \in \mathcal{X}$ and data augmentation $A_i$ with parameter space $\Theta(A_i)$, the \textbf{augmentation neighborhood} of $x_0$ induced by $A_i$ is defined as:
    \begin{equation}
        A_i(x):=\bigcup\limits_{\theta\in\Theta(A_i)}A_i(\theta,x).
    \end{equation}
\end{definition}

Now we should add some restrictions to this set so make it "mild".

At first we introduce the conception $C$, a map from input space $\mathcal{X}$ to the label space $\mathscr{L} = \{c_1,\dots,c_l\}$ where $l$ denotes the number of class, such that for every clean sample pair $(x,y)\sim p(x,y)$, $C(x) = y\in \mathscr{L}$, conception is the desired ground truth map. $C$ induces a partition of the sample space, by giving $l$ mutually disjoint sets such that $\Gamma_i = \{x|C(x) = c_i\}$, what we call level set. We denote the level set of the class of a sample $x_0$ with $\Gamma_{x_0}$, and we use this level set to describe the semantic consistency. The conception $C$ represents the prior knowledge of people when they perform data augmentation. The definition is given as followed:

\begin{definition}\label{definition: CANofSingleAug}
For any given clean sample $x_0 \in \mathcal{X}$, and augmentation $A_i$ with parameter space $\Theta(A_i)$, the \textbf{consistency augmentation neighborhood} (CAN for short) of $x_0$ induced by $A_i$ is defined as :
    \begin{equation}
        \mathcal{O}_{x_0}^{A_i}:=A_i(x_0)\cap\Gamma_{x_0}.
    \end{equation}
\end{definition}

Now we will introduce how to sample from the CAN.

\subsubsection{Sampling from CAN of $x_0$}
\label{section:Sampling}
An augmented sample is generated given a clean sample, together with the aforementioned mild argument, we claim that the sampling procedure should be described with a conditional distribution $p(x'|x)$, whose supporting set is CAN of $x_0$. 
The fact that $\forall x' \in \mathcal{O}_{x_0}^{A_i}$, 
there exists only one $\pmb{\theta}: = h_{A_i,x_0}^{-1}(x') \in \Theta(A_i)$ such that $x' = A_i(\pmb{\theta},x)$ which is ensured by our definition. 
Furthermore, with the measurability of $A_i(\cdot,x)$, $x'$ is a random variable. 
Therefore, for any given data augmentation $A_i$, the conditional distribution $p(x'|x)$ induced by $A_i$ is defined as:

\begin{definition}\label{definition: ConditionalSingle}
For any given clean sample $x\sim p(x)$,  the conditional distribution $p(x'|x)$ of the adjoint variable $x'$ with $\mathrm{supp}(p(x'|x)) = \mathcal{O}_{x}^{A_i}$ is given as
\begin{equation}
    p(x'|x) \propto p_i(h_{A_i,x}^{-1}(x'))\left|\frac{\partial}{\partial\pmb{\theta}}A_i(\pmb{\theta},x)\right|_{\pmb{\theta} = h_{A_i,x}^{-1}(x')}\mathbf{1}_{x'\in \mathcal{O}_{x}^{A_i}}.
\end{equation}
\end{definition}

Sampling from $A_i(x)$ is equivalent to sampling from $p_i(\pmb{\theta})$ defined on $\Theta(A_i)$, 
for $h_{A_i,x}^{-1}(A_i(x)) = \Theta(A_i)$, given the injectivity of $A(\cdot,x)$. 
Furthermore, to sample from $A_i(x)\cap \Gamma_{x}$, we need to sample from the truncated distribution: 
\begin{equation}
\label{truncate_prior}
    p_i(\pmb{\theta})\mathbf{1}_{\pmb{\theta}\in h_{A_i}^{-1}(\mathcal{O}_{x}^{A_i})}.
\end{equation}

Rejection sampling is one effective way to generate augmented samples, but it may be infeasible in high-dimensional cases due to its computational cost. Although various methods, such as nested sampling, adaptive multilevel splitting, or sequential Monte-Carlo sampling, could be viable alternatives, we leave the exploration of these methods for future work. Additionally, the rejection step can be seen as a way to inject humane prior knowledge to samples, which aligns with the intuition on the process of data augmentation. In our experiment, we assume that it would be enough to sample from the subset of $h_{A_i}^{-1}(\mathcal{O}_{x}^{A_i})$, we use human prior knowledge in rejection sampling to roughly determine a subset of it. We begin by selecting the candidate of edges of these subsets, then apply $A(\theta,x)$ for parameters of these edges, and reject or accept these edges by observing the output samples. However, this method is inefficient and risky, rejection sampling is infamous for its inefficiency and the initial selection of edges could be problematic since they may be too small compared to the ground truth. We plan to develop better methods based on our framework in future work.

The conditional distribution $p(x'|x)$ is now well-defined, with its marginal distribution given by $p^*(x') = \int p(x'|x)p(x)\mathrm{d}x$. However, it is important to note that $p(x'|x)$ is unlikely to be tractable. The description above is useful in understanding that an augmented sample is a random variable induced from the of data augmentation, given the measurability of $A(x,\cdot)$.

Finally, it's worth mentioning that generating $M$ samples for each of the $N$ clean samples does not result in $M\times N$ completely independent augmented samples. But \eqref{emp_aug_risk} still yields an unbiased estimator of the shifted population risk, due to the following equation:
\begin{align*}
    \mathbb{E}_{p(x'|y)}\left[\mathcal{L}(y,f(x'))\right] &= \mathbb{E}_{p(x|y)}\left[\left.\mathbb{E}_{p(x'|x,y)}\left[\mathcal{L}\left(y,f(x')\right)\right|x\right]\right],\\
    &=\mathbb{E}_{p(x|y)}\left[\left.\mathbb{E}_{p(x'|x)}\left[\mathcal{L}\left(y,f(x')\right)\right|x\right]\right],\\
 \hat{R}_f(p^*) & = \frac{1}{N}\sum_{i=1}^N \frac{1}{M}\sum_{j=1}^M \mathcal{L}(y_i,f(x'_{i_j})),
\end{align*}
Taking expectation on the both side of the third equation yields the desired result. One should notice that augmented sample $x'$ is independent of $y$ once its original clean sample $x$ is given, which explains the second equality.

The above definition in the case of a finite set of data augmentations and the composition order is given in Appendix \ref{Supplementary Definition}.

By establishing these definitions and concepts in this section, we are provided with a comprehensive understanding of the topic at hand. Which provides a solid foundation for the decomposition of the expected risk in the coming section.

\subsection{The Artificial Shifted Population Risk}
\label{ASPR}
After defining the augmented neighborhood and giving sampling method by defining the adjoint variable $x'$ and its conditional distribution $p(x'|x)$, 
we then evaluate the risk on the shifted population $p(x',y)$. 
One should realize that the collection of all the samples generated from $p(x)$ is a subset of the samples generated from $p^*(x')$. 

For simplicity, we only consider the risk function in the case of cross-entropy and softmax on the shifted population $p(x',y)$, and our method should be able to extend to the other cases similarly:

\begin{equation}\label{equation: RiskonShift}
\begin{aligned}
    R_f(p^*) &= \mathbb{E}_{p^*(x')}\left[H(p(y|x'),q_\phi(y|x'))\right],\\
    & = \mathbb{E}_{p(y)p(x'|y)}[-\ln q_{\phi}(y|x')]\\
    & = \mathbb{E}_{p(y)p(x',x|y)}[-\ln q_{\phi}(y|x')],
\end{aligned}
\end{equation}
among which $\phi$ denotes the parameter of the neural network and $q$ represents a probabilistic surrogate model. The decomposition of this shifted population risk is examined with the following theorem:
\begin{theorem}
\label{equation: ShiftDecomp}
With the shifted population risk in the form of \eqref{equation: RiskonShift}, we have the following decomposition:
\begin{equation}
\begin{aligned}
    \mathbb{E}_{p^*(x')}&\left[H(p(y|x'),q_\phi(y|x'))\right] \\
    &= \mathbb{E}_{p(x)}\left[H(p(y|x),q_{\phi}(y|x))\right] +\mathbb{E}_{p(x)p(y|x)p(x'|x)}\left[\ln\frac{q_{\phi}(y|x)}{q_{\phi}(y|x')}\right],
\end{aligned}
\end{equation}
\end{theorem}
The proof of the Theorem \ref{equation: ShiftDecomp} can be found in Appendix \ref{prof equation: ShiftDecomp}.
This demonstrates that in the case of cross-entropy and softmax, 
the \textbf{shifted population risk} is actually the sum of the \textbf{original population expected risk} and a gap term that can be viewed as a \textbf{consistency regularization term}. 
Next, we provide a theorem that explains the second term.

\subsection{Understanding the decomposition of shifted population risk}
From the last section, we have:
\begin{equation}
\begin{aligned}
\mathbb{E}_{p(x)p(y|x)p(x'|x)}\left[\ln\frac{q_\phi(y|x)}{q_\phi(y|x')}\right]=\mathbb{E}_{p(x)p(x'|x)}\left[\ln\frac{q_\phi(y_x|x)}{q_\phi(y_x|x')}\right],
\end{aligned}
\end{equation}
where $y_x$ is the ground true label of clean sample $x$. Since $q_\phi(y|x)$ is modeled with softmax, we have:

\begin{equation}
q_\phi(y_i|x) = \frac{\exp(\mathbf{w}_i^Th_\theta(x))}{\sum_{j=1}^l\exp(\mathbf{w}_j^Th_\theta(x))},
\end{equation}

where $h_\theta(x) = (h_1(x),h_2(x),\dots,h_d(x),\dots,h_D(x))^T$ (the subscript $\theta$ of the component is omitted for convenience) is the feature vector of $x$, and $\mathbf{W}= (\mathbf{w}_1,\dots,\mathbf{w}_l)$ is the weight of the output layer, now $\phi = \{\theta,\mathbf{W}\}$. For every feature $h_d(x)$, its density is:
\begin{equation}
    q_\phi(y_i|h_d(x)) = \frac{\exp( w_{i,d}h_d(x))}{\sum_{j=1}^l \exp(w_{j,d}h_d(x))},
\end{equation}
Inspired by \cite{16he2019data}, we partition the features into major features and minor features by information gains. For major features, the density function $q_\phi(y|h_d)$ concentrates on some point mass. For minor features, the possibility density $q_\phi(y|h_d)$ is relatively uniform.

Then for every given $x$, we have:
\begin{equation}
\begin{aligned}
\mathbb{E}_{p(x'|x)}\left[\ln\frac{q_\phi(y_x|x)}{q_\phi(y_x|x')}\right]= \mathbb{E}_{p(x'|x)}\left[\ln\left(\frac{\sum_{j=1}^l\exp\left((\mathbf{w}_j-\mathbf{w}_x)^Th_\theta(x')\right)}{\sum_{j=1}^l\exp\left((\mathbf{w}_j-\mathbf{w}_x)^Th_\theta(x)\right)}\right)\right],\\
\end{aligned}
\end{equation}
For convenience, we denote
\begin{equation}
\begin{gathered}
\exp\left((\mathbf{w}_j-\mathbf{w}_x)^Th_\theta(x)\right) = \rho_{\theta,x,j},\\
\sum_{j=1}^l \rho_{\theta,x,j}= \rho_{\theta,x},\\
\end{gathered}
\end{equation}
we then examine the relationship of feature and the second term with the following theorem:
\begin{theorem}\label{theorem: OrderofSecondterm}
Assuming that for every $\theta$, sample pair $(x,x')$ and indicies $j$, there exist $\beta_{1,j},\alpha_{1,j} >0 $ such that
\begin{equation}
\begin{gathered}
\alpha_{1,j}<\rho_{\theta,x,j}, \ \rho_{\theta,x',j} < \beta_{1,j},\\
\end{gathered}
\end{equation}
Then for any given $x$, we have:
\begin{equation}
\begin{aligned}
\mathbb{E}_{p(x'|x)}\left[\left|\ln \frac{q_\phi(y_x|x)}{q_\phi(y_x|x')}\right| \right] = \mathbb{E}_{p(x'|x)}\left[\left|\sum_{j=1}^l O\left((\mathbf{w}_j-\mathbf{w}_x)^T\left(h_\theta(x)-h_\theta(x')\right)\right) \right|\right],
\end{aligned}
\end{equation}
\end{theorem}

The proof of the Theorem \ref{theorem: OrderofSecondterm} can be found in the Appendix \ref{prof theorem: OrderofSecondterm}.
With Theorem \ref{theorem: OrderofSecondterm}, we show how the second term affects the weights. 
Since the data augmentation must cause a large variance in some features particularly in early training phases, which means that
\begin{equation}
\exists \ \eta_1>0, \ \left|h_d(x) -h_d(x')\right|>\eta_1,
\end{equation}
for some features including minor and major features. This forces that $\forall j\in \{1,\dots,l\}$, $w_{j,d}\to w_{x,d} $, resulting in a uniform distribution of $q_\phi(y|h_d(x))$, and such regularization of $w_{j,d}$ is not appropriate for major features. Now let us see how the first term affects the weights
\begin{equation}
\begin{aligned}
\mathbb{E}_{p(x)}\left[H(p(y|x),q_\phi(y|x))\right] = \mathbb{E}_{p(x)}\left[\sum_{i=1}^l\exp\left((\mathbf{w}_j-\mathbf{w}_x)^Th_\theta(x)\right)\right],
\end{aligned}
\end{equation}
And for any minor feature, its variation should not change the result, hence we have $w_{j,d}\approx w_{i,d}$, $1\leq i,j\leq l$. In contrast, the weights of major features should be different: 
\begin{equation}
\exists \ \eta_2 > 0, \ \left|w_{j,d}-w_{x,d}\right|>\eta_2
\end{equation}
now we realize that, with the effect of data augmentation, the first term and the second term have different impacts on the weight of some major features and the same impact on minor features. Since our model mainly relies on major features to provide prediction, such an effect causes an unstable convergence. To highlight the positive effect provided by the first term at the beginning, a simple trick is to add a coefficient $\lambda$ ($\lambda < 1$) to the second term.  




Now we discuss how $\lambda$ may help refine the generalization of the model. We denote the model trained using augmented samples as $f_{aug}$: 
\begin{equation}
\begin{aligned}
R_{f_{aug}}(p^*)  &= \mathbb{E}_{p(y)p(x'|y)}\left[\mathcal{L}(y,f(x'))\right],\\
R_{f_{aug}}(p) &= \mathbb{E}_{p(y)p(x|y)}\left[\mathcal{L}(y,f(x))\right],
\end{aligned}
\end{equation}
Note that we train our model using $\hat{R}_{f_{aug}}(p^*)$ and evaluate the generalization of our model using $R_{f_{aug}}(p)$. Based on the assumption \ref{generalization bound assumption}, with augmented sample and clean sample pairs instead of clean samples alone, we have:
\begin{equation}\label{equation: ERMinequality}
R_{f_{\mathrm{aug}}}(p^*)\leqslant\hat{R}_{f_{\mathrm{aug}}}(p^*)+\phi(C(f),B(N \times M),\delta),
\end{equation}
 Theorem \ref{equation: ShiftDecomp} can then be reformulated with our new formulation:
\begin{equation}\label{eq:unify_upper_bdd}
\begin{aligned}
R_{f_{\mathrm{aug}}}(p^*) &= R_{f_{\mathrm{aug}}}(p) + \mathrm{GAP},\\
\hat{R}_{f_{\mathrm{aug}}}(p^*) &= \hat{R}_{f_{\mathrm{aug}}}(p) +\widehat{\mathrm{GAP}}_{M\times N},
\end{aligned}
\end{equation}
where $\mathrm{GAP}$ is the second term in the right hand side of Theorem \ref{equation: ShiftDecomp} and $\widehat{\mathrm{GAP}}_{M\times N}$ is its empirical estimator using $M\times N$ non iid pairs of $(x,x')$. Hence, \eqref{equation: ERMinequality} is reformulated by:
\begin{equation}\label{eq:key_equality}
\begin{aligned}
R_{f_{\mathrm{aug}}}(p)&\leqslant\hat{R}_{f_{\mathrm{aug}}}(p)+\phi(C(f),B(N \times M),\delta)+ \\
&\widehat{\mathrm{GAP}}_{M\times N}-\mathrm{GAP}.
\end{aligned}
\end{equation}
Now we show that the noise $\widehat{\mathrm{GAP}}_{M\times N}-\mathrm{GAP} \to0 $:

$$
\mathrm{GAP} = \mathbb{E}_{p(x)p(y|x)p(x'|x)}\left[\ln\frac{q_{\phi}(y|x)}{q_{\phi}(y|x')}\right],
$$
we denote $\mathcal{B}(y,g(x,x')) = \ln\frac{q_{\phi}(y|x)}{q_{\phi}(y|x')}$, then we assume that given any clean sample pair $(x_i,y_i)$:
$$
\mathrm{Var}_{p(x'|x_i)}\left[\mathcal{B}(y_i,g(x_i,x'))\right] \leq B,
$$
then for the estimator:
$$
\begin{aligned}
&\mathbb{E}_{p(x,y)p(x'|x)}\left[\mathcal{B}(y,g(x,x'))\right] \\&=\mathbb{E}_{p(x,y)}\left[\left.\mathbb{E}_{p(x'|x)}\left[\mathcal{B}\left(y,g(x,x')\right)\right|x,y\right]\right],\\
&\widehat{\mathrm{GAP}}_{M\times N} = \frac{1}{N}\sum_{i=1}^N \frac{1}{M}\sum_{j=1}^M \mathcal{B}(y_i,g(x_i,x'_{i_j})),
\end{aligned}
$$
where $x'_{i_j}$ denotes augmented samples from $p(x'|x_i)$ consider its variance:
$$
\begin{aligned}
\mathrm{Var}\left(\widehat{\mathrm{GAP}}_{M\times N}\right) &=\frac{1}{N^2M}\sum_{i=1}^N \mathrm{Var}_{p(x'|x_i)}\left[\mathcal{B}(y_i,g(x_i,x'))\right],
\end{aligned}
$$
with the assumption:
$$
\mathrm{Var}\left(\widehat{\mathrm{GAP}}_{M\times N}\right) \leq\frac{B}{NM},
$$
then the variance is of order $O(1/NM)$, which indicates the faster convergence speed.

We determine that the generalization of model depends on $\hat{R}_{f_{\mathrm{aug}}}(p)$ instead of what we directly optimize: $\hat{R}_{f_{\mathrm{aug}}}(p^*)$. 
Hence we would like to keep the consistency between $\hat{R}_{f_{\mathrm{aug}}}(p) $ and  $\hat{R}_{f_{\mathrm{aug}}}(p^*)$, 
$i.e.$, the decreasing of $\hat{R}_{f_{\mathrm{aug}}}(p^*)$ guarantees that of $\hat{R}_{f_{\mathrm{aug}}}(p) $ to ensure the improvement of generalization when training the model. 
As it is analyzed before, $\widehat{\mathrm{GAP}}_{M\times N}$ may lead to different weights of some major features compared with $\hat{R}_{f_{\mathrm{aug}}}(p)$ in early training stages, 
which will destroy such consistency. This indicates the importance of our proposed coefficient $\lambda$.

\section{Experiment}
\label{others}
We demonstrate the standard training strategy in Algorithm \ref{algorithm: standardAlg} 
and our proposed training strategy in Algorithm \ref{algorithm: OurAlg} in Appendix \ref{AppendixPseudoCode}. We also conduct an experiment on the selection of the hyperparameter $\lambda$ of Algorithm \ref{algorithm: OurAlg} in Appendix \ref{choose_lambda}.

\subsection{Experiment Implementation}
\paragraph{Standard Scenario Experiment: Validation Models and Datasets} We have conducted experiments on CIFAR10/100 \cite{krizhevsky2009learning}, 
Food101 \cite{bossard2014food}, and ImageNet (ILSVRC-2012) \cite{russakovsky2015imagenet} with various models to evaluate our training strategy. 
For each of them, a validation set is split from the training set to find networks with the best performances. More dataset splitting details are shown in Appendix \ref{DSD}.
In this paper, ResNet \cite{03he2016deep} and WideResNet \cite{BMVC2016_87} are trained with different strategies. 
For datasets CIFAR10/100 and Food101, ResNet-18, ResNet-50, WideResNet-28-10 and WideResNet-40-2 are chosen as our baseline models. 
For ImageNet, ResNet-50 and ResNet-101 are used for evaluation. 
All images in baseline (standard method) and our method are processed with same augmentation (horizontal flips, random crops and random rotation). 
$\lambda$ was selected to 0.5 for it achieve the best performance among all the experiments with our strategy. For a fair comparison, we set the basic batch size (bbs) and performed standard method experiments with both 1x bbs and 2x bbs (our method actually takes twice the amount of data sample) to ablate the estimation error effect caused by the batch size. More details about data augmentation and network training are shown in Appendix \ref{D_aug} and Appendix \ref{T_Detail}.

To ensure that our strategy is applicable to other settings, we conduct experiments in the following two cases:

\paragraph{OOD Scenario Experiment: Validation Models and Datasets} 
Experiments on PACS \cite{li2017deeper} are conducted using ResNet-18 and ResNet-50 \cite{03he2016deep}. In these experiments, we employed the leave-one-domain-out strategy for OOD validation. For image augmentation, we followed the same approach as Domainbed \cite{gulrajani2020search}, both in the ERM algorithm and our proposed method. Further information regarding data augmentation and network training can be found in Appendix \ref{D_aug} and Appendix \ref{T_Detail}.

\paragraph{Long-Tailed Scenario Experiment: Validation Models and Datasets}
We consider long-tail (LT) imbalance and conducted experiments on LT CIFAR-10 \cite{cui2019class} using ResNet-18. 
We keep the validation set and test set unchanged and reduce the number of training set per class according to the function $n=n_i\mu^i$, where $n_i$ is the original number of the $i-th$ class of the training set (following \cite{cui2019class}). $\mu$ is between 0 and 1, which is determined by the number of training samples in the largest class divided by the smallest. This ratio is called imbalance ratio and it is set from 10 to 100 in our settings. Further information regarding training hyperparameters can be found in Appendix \ref{D_aug} and Appendix \ref{T_Detail}.

\subsection{Experimental Results}\label{sec:exp_result_e}

\paragraph{Settings and instructions} For standard scenario experiment, we select the model with the highest validation accuracy during training and report the test accuracy in Figure \ref{fig:Standard Experiment}. The results with error bars are presented at Appendix \ref{error_bar}, where we have conducted three independent experiments and calculated the mean values as the results on CIFAR10/100 and Food101 and only one independent experiment on ImageNet (ILSVRC2012) because of computational constraints.

As for the OOD scenario experiments, we have conducted three independent experiments and select the model with the best top-1 accuracy on the test domain. The results with error bar can be seen in Figure \ref{fig:pacs} and Appendix \ref{error_bar} Table \ref{tab:pacs}.

For long-tailed scenario experiment, we use the Area Under the Curve (AUC), Average Precision (AP) and top-1 accuracy as evaluation metrics. We select the model with the best AUC on the validation set during training and report the results on the test set in Figure \ref{fig:LT Experiment} and Appendix \ref{error_bar} Table \ref{tab:LT}.

\begin{figure*}[t]
    \centering
    {
        \label{cifar10}
        \includegraphics[width=0.23\linewidth]{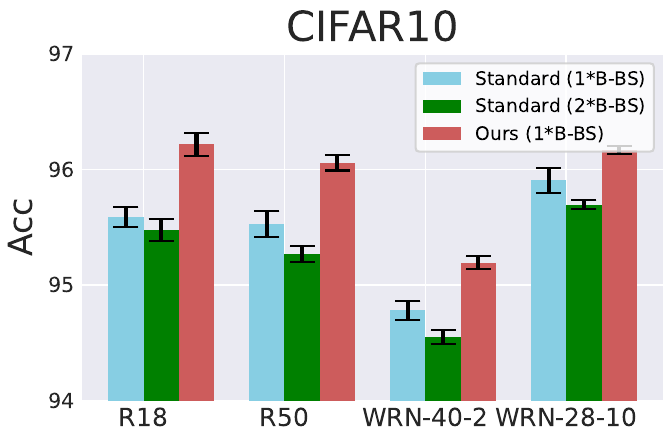}
    }
    {
        \label{cifar100}
        \includegraphics[width=0.23\linewidth]{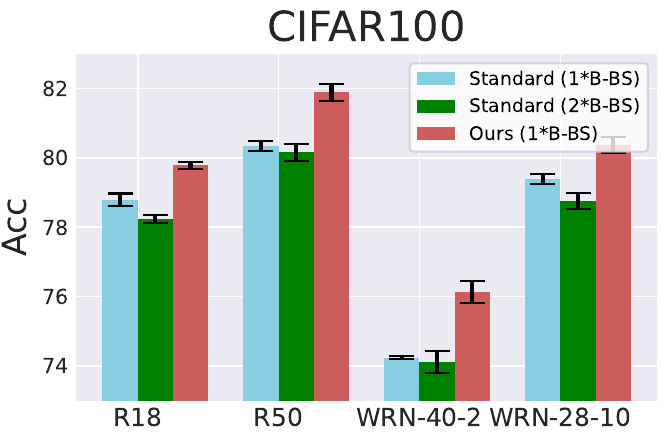}
    } 
    {
        \label{Food101}
        \includegraphics[width=0.23\linewidth]{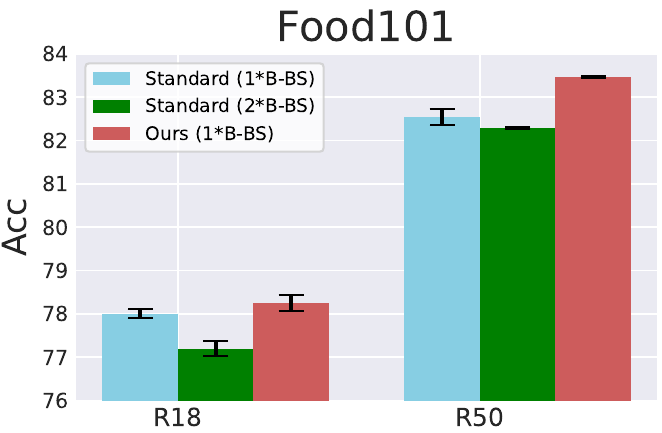}
    }
    {
        \label{ImageNet}
        \includegraphics[width=0.23\linewidth]{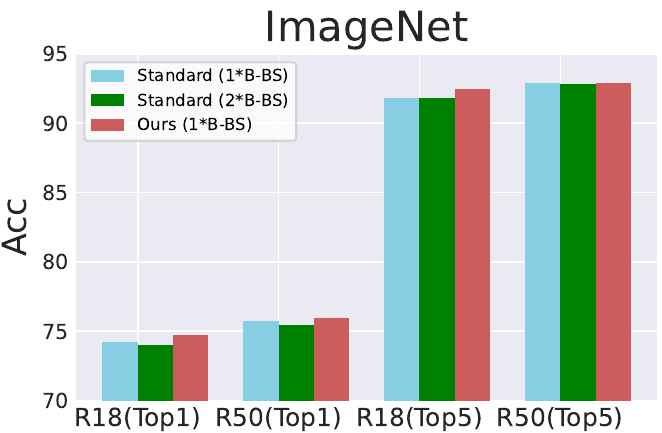}
    }
    \caption{Top-1 accuracy(\%) with error bar (mean$\pm$std) on CIFAR10/100, Food101 and ImageNet  on the test set. The Y-axis is the Top-1 accuracy and the X-axis is the type of network.}
    \label{fig:Standard Experiment}
\end{figure*}

\begin{figure}[t]
    \centering
    {
        \label{OODR18}
        \includegraphics[width=0.475\linewidth]{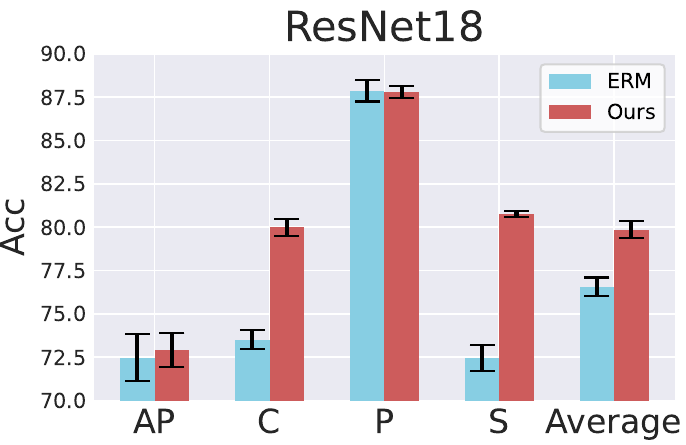}
    }
    {
        \label{OODR50}
        \includegraphics[width=0.475\linewidth]{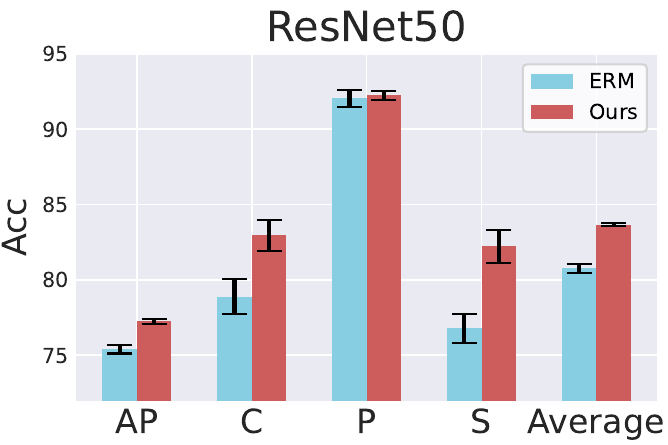}
    } 
    \caption{Top-1 accuracy(\%) with error bar (mean$\pm$std) over the four test domain of PACS and their average. The X-axis is test-domain and the Y-axis is the Top-1 accuracy.}
    \label{fig:pacs}
\end{figure}

\begin{figure}[t]
    \centering
    {
        \label{LT_AUC}
        \includegraphics[width=0.31\linewidth]{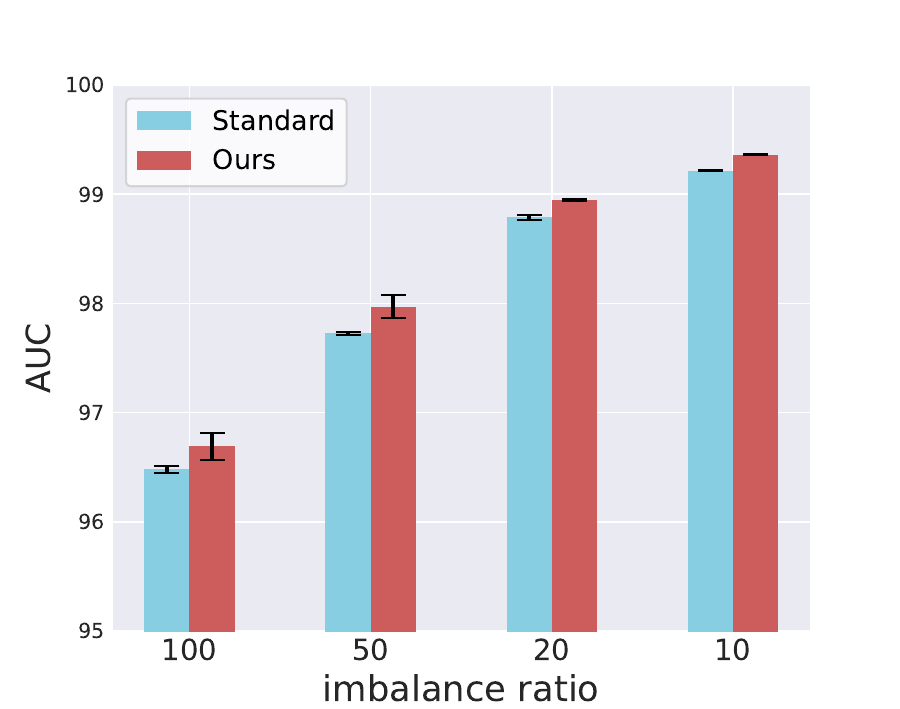}
    }
    {
        \label{LT_ACC}
        \includegraphics[width=0.31\linewidth]{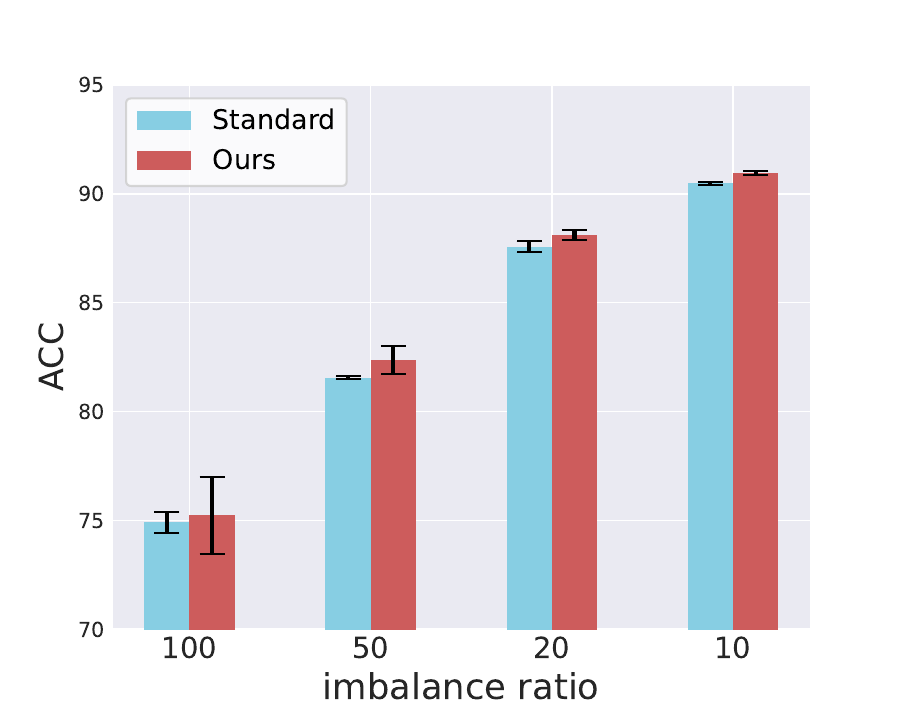}
    } 
    {
        \label{LT_AP}
        \includegraphics[width=0.31\linewidth]{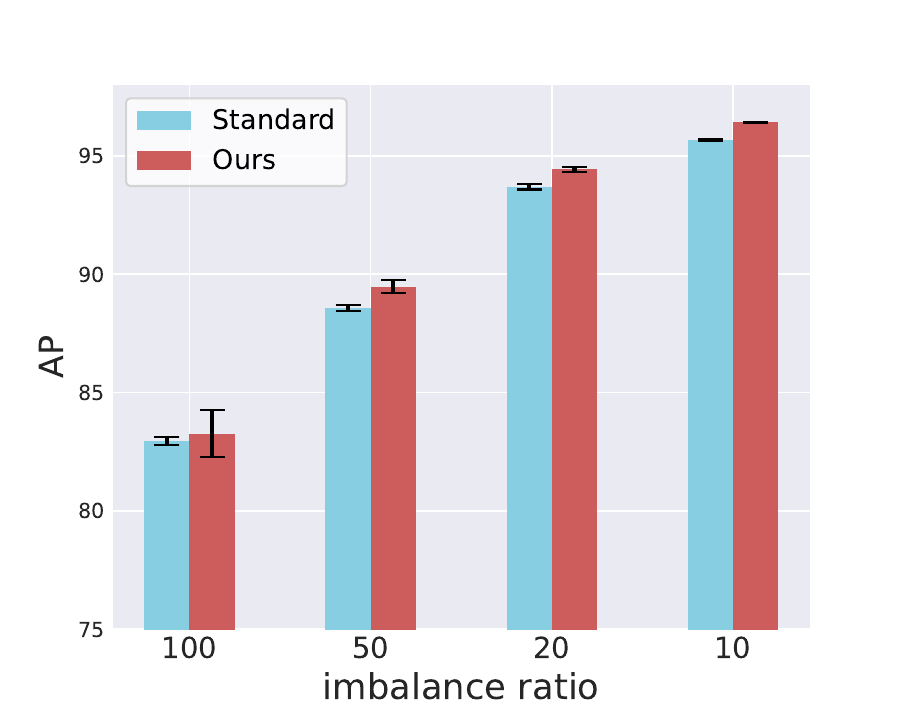}
    }
    \caption{AUC, Top-1 accuracy (\%), AP with error bar (mean$\pm$std) of Resnet-18 on the long-tailed scenario experiment (LT-CIFAR10). The X-axis of the figures is the value of the imbalance ratio.}
    \label{fig:LT Experiment}
\end{figure}




    

\paragraph{Experiment Analysis}
From our experimental results (Figure \ref{fig:Standard Experiment}, Figure \ref{fig:pacs}, Figure \ref{fig:LT Experiment}), we can see that the model trained with our proposed consistency regularization strategy of data augmentation converges to a better local optimum. 

From Figure \ref{Figure 3}, 
we can see the validation set performance of our method even exceeds the training set performance of the standard data augmentation training method in 
almost the whole process of training. 
This demonstrates the improvement of generalization after adding the coefficient. 

As demonstrated in Figure \ref{Fcc} and Figure \ref{xunliantuxiang}, our training strategy leads to a stable convergence compared with the standard data augmentation training strategy. 
The stable convergence is caused by the coefficient $\lambda $, as we discuss in section \ref{ASPR}. 
The coefficient $\lambda $ diminishes the negative effect of estimate variance, resulting in a more stable convergence. The training process for all circumstances is presented in Appendix \ref{appendix:accuracycurve}

\section{Conclusion and Discussion}


\paragraph{Rethinking of Shifted Population}
In this paper, we develop a new set of definitions for shifted population, augmented samples and its conditional distribution. 
We leverage our proposed definition to establish the decomposition of the shifted population risk, providing an explanation for how data augmentation enhances the generalization ability of model.
\paragraph{Better Training Strategy}
Based on the proposed decomposition, 
we realize that the key to improving generalization lies in keeping the consistency between $\hat{R}_{f_{\text{aug}}}(p^*)$ and $\hat{R}_{f_{\text{aug}}}(p)$, 
which is likely to be violated by the gap term specifically in the early training stages. Adding a coefficient to the gap term refines this, and it is proposed as a training strategy with augmentation. As demonstrated in our experiment, 
our method outperforms the standard augmentation training strategy. Meanwhile, our proposed strategy is highly related to the augmentation schedule, an existing training strategy. Our work could provide comprehensive understanding on how it works. What's more, there is more than one solution to the problem of the gap term, which is left for future work.
\paragraph{Limitation}
Considering the fact that this paper mainly conducts analysis in the case of classification tasks, 
some of the results proposed in this paper lack versatility. However, the framework of the analysis is transferable, 
and based on the definition of expected risk, similar results can be attained on other tasks. 
Conditional distribution of adjoint variable $p(x'|x)$ is intractable given the fact that although the differentiability of most of the classic augmentations has been verified in other works, 
there are data augmentations that have not, some of them may even be not genuinely differentiable. 
Hence, other definitions of $p(x'|x)$ that bypass the necessity of differentiability can be explored in future work.

\begin{credits}
\subsubsection{\ackname} We sincerely thank all reviewers for their efforts to improve the quality of this paper. This work was supported in part by the Guangdong Basic and Applied Basic Research Foundation (project code: 2024A1515011896, 2023A1515012943, and 2022A1515110568), and the Guangzhou Basic and Applied Basic Research Foundation (project code: 2023A04J1683) and Technological Innovation Strategy of Guangdong Province, China (project code: pdjh2022a0030) and Guangdong Province College Students PanDeng Project (project code: 202210561138).

\subsubsection{\discintname}
The authors have no competing interests to declare that are
relevant to the content of this article.
\end{credits}
%

\clearpage
\appendix

\section{Supplementary Definition}
\label{Supplementary Definition}
In this section we introduce the CAN of $x_0$ together with the conditional distribution defined on it induced by a finite set of augmentations $\{A_1,\dots A_m\}$ with parameter space $\Theta(A_1),\dots\Theta(A_m)$ for a given composite order $\sigma$. Realizing that with the given order of composition, there is a new data augmentation:
\begin{equation}
\begin{aligned}
    A_{i_m}\circ\dots\circ A_{i_1}:(\Theta(A_{i_m})\times\dots\times \Theta(A_{i_1}))\times \mathcal{X}\rightarrow \mathcal{X}\\
((\pmb{\theta}_{i_m},\dots\pmb{\theta}_{i_1}),x_0) \mapsto A_{i_m}(\pmb{\theta}_{i_m})\circ\dots\circ A_{j_1}(\pmb{\theta}_{i_1},x)
\end{aligned}
\end{equation}
Note that $\forall\ x' \in \mathcal{X}, \pmb{\theta}_i \in \Theta(A_i),\pmb{\theta}_j \in \pmb{\Theta}(A_j),A_i(\pmb{\theta}_i)\circ A_j(\pmb{\theta}_j,x') = A_i(\pmb{\theta}_i,A_j(\pmb{\theta}_j,x'))$. 

The augmentation neighborhood of $x_0$ induced by $\mathcal{A}$ with composite order $\sigma$ is:

\begin{definition}\label{definition: MultipleAugNeib}
For any given clean sample $x \in \mathcal{X}$, and the finite set of augmentations  $\mathcal{A} = \{A_1,\dots A_m\}$ with parameter space , for a given composite order  $\sigma$ such that $\sigma(1,\dots,m) = (j_1,\dots,j_m)$, the \textbf{augmentation neighborhood of} \boldmath{$x_0$} induced by $\mathcal{A}$ for a given composite order  $\sigma$ is defined as:
\begin{equation}
        A^{\sigma}(x)=\bigcup\limits_{(\theta_{j_1},\theta_{j_2},\ldots,\theta_{j_m}) \in \atop \Theta(A_{j_1}) \times \Theta(A_{j_2})\times\ldots\times\Theta(A_{j_m})} (A_{j_m}(\theta_{j_m})\circ\cdots\circ A_{j_2}(\theta_{j_2})\circ A_{j_1}(\theta_{j_1},x))
    \end{equation}
\end{definition}

Although there are $m!$ different ways of composition and composite order matters, but for simplicity, leave the matter to future work. We usually omit this superscript and denote it as $\mathcal{A}(x)$.

Now the CAN of $x_0$ induced by $\mathcal{A}$ with composite order $\sigma$ is:

\begin{definition}\label{definition: CANofMultiAug}
For any given clean sample $x \in \mathcal{X}$, and the finite set of augmentations  $\mathcal{A} = \{A_1,\dots A_m\}$ with parameter space $\Theta(A_1),\dots\Theta(A_m)$ and a given order of composition $\sigma$, the \textbf{consistent augmentation neighborhood} (CAN for short) of $x_0$ induced by $\mathcal{A}$ is defined as :
    \begin{equation}
        \mathcal{O}_{x}^{A,\sigma}:=A^{\sigma}(x)\cap\Gamma_{x}
    \end{equation}
\end{definition}
For the same reason that the permutation $\sigma$ is assumed to be given, part of the superscript is omitted and can be simplified as $\mathcal{O}^{\mathcal{A}}_{x}$.

When we need to sample in the CAN of $x$ induced by a finite set of augmentation $\mathcal{A}$ for a given order of composition $\sigma$. For any $x'\in \mathcal{A}^{\sigma_p}(x)$, there exists only one $\pmb{\theta} = h^{-1}_{A_{i_m}\circ\dots\circ A_{i_1},x}(x') = (\pmb{\theta}_{i_m},\dots\pmb{\theta}_{i_1})$. Since $(\pmb{\theta}_{i_m},\dots\pmb{\theta}_{i_1})$ are mutually independent, the conditional distribution $p(x'|x)$ induced by it with a supporting set on $\mathcal{O}_x^{\mathcal{A}}$ is defined as:

\begin{definition}\label{ConditionalMultiple}
For a given clean sample $x\sim p(x)$,  the conditional distribution $p(x'|x)$ of the adjoint variable $x'$ with $\mathrm{supp}(p(x'|x)) = \mathcal{O}_{x}^{\mathcal{A}}$ is given as
\begin{equation}
\begin{aligned}
    & p(x'|x)\\
    & = \frac{p_{i_1}(\pmb{\theta}_{i_1})\dots p_{i_m}(\pmb{\theta}_{i_m})\left|\frac{\partial A(\pmb{\theta}_{i_m})\circ\dots \circ A(\pmb{\theta}_{i_j})\circ\dots \circ A(i_1,x))}{\partial \pmb{\theta}_{i_m}\dots \pmb{\theta}_{i_j}\dots\pmb{\theta}_{i_1}}\right|_{\pmb{\theta} = h^{-1}_{A_{i_m}\circ\dots\circ A_{i_1},x}(x')}}{Z_3}\mathbf{1}_{x'\in \mathcal{O}^{\mathcal{A}}_x}
\end{aligned}
\end{equation}
\end{definition}

$Z_3$ is the normalization constant since it is limited on the level set of $x_0$, and the sampling method is similar to that when single augmentation is considered. 
\section{Proof of Theorems}

\subsection{Proof of Theorem \ref{equation: ShiftDecomp}}
\label{prof equation: ShiftDecomp}
\begin{theorem}{equation: ShiftDecomp}
In the case of cross-entropy and softmax, the shifted population risk has the following decomposition:
\begin{equation}
    \begin{aligned}
        & \mathbb{E}_{p^*(x')}\left[H(p(y|x'),q_\phi(y|x'))\right]\\
        & = \mathbb{E}_{p(x)}\left[H(p(y|x),q_{\phi}(y|x))\right]+\mathbb{E}_{p(x)p(y|x)p(x'|x)}\left[\ln\frac{q_{\phi}(y|x)}{q_{\phi}(y|x')}\right]
    \end{aligned}
\end{equation}

\end{theorem}
\begin{proof}
\begin{equation}
\begin{aligned}
&\mathbb{E}_{p(y)p(x',x|y)}[-\ln q_{\phi}(y|x')]\\
&= \mathbb{E}_{p(y)p(x'|x)p(x|y)}[-\ln q_{\phi}(y|x')]\\&=\mathbb{E}_{p(x)p(y|x)p(x'|x)}\left [-\ln q_{\phi}(y|x')\right]\\
&=\mathbb{E}_{p(x)p(y|x)p(x'|x)}\left[\ln\frac{1}{q_\phi(y|x)}+\ln\frac{q_{\phi}(y|x)}{q_{\phi}(y|x')}\right]\\
&=\mathbb{E}_{p(x)p(x'|x)p(y|x)}[-\ln q_{\phi}(y|x)]+
\mathbb{E}_{p(x)p(x'|x)p(y|x)}\left[\ln\frac{q_{\phi}(y|x)}{q_{\phi}(y|x')}\right]\\
&=\mathbb{E}_{p(x)p(y|x)}[-\ln q_{\phi}(y|x)]+
\mathbb{E}_{p(x)p(x'|x)p(y|x)}\left[\ln\frac{q_{\phi}(y|x)}{q_{\phi}(y|x')}\right]\\
&= \mathbb{E}_{p(x)}\left[H(p(y|x),q_{\phi}(y|x))\right]+\mathbb{E}_{p(x)p(y|x)p(x'|x)}\left[\ln\frac{q_{\phi}(y|x)}{q_{\phi}(y|x')}\right]
\end{aligned}
\end{equation}
One should note that the first equation holds because $p(x',x|y)$ = $p(x'|x,y)p(x|y) $ is $ p(x'|x)p(x|y)$ since $x'$ is a random variable that independent of $y$ after $x$ is given.
\end{proof}
\subsection{Proof of Theorem \ref{theorem: OrderofSecondterm}}
\label{prof theorem: OrderofSecondterm}
\begin{theorem}{theorem: OrderofSecondterm}
Assuming that for every $\theta$ and every $(x,x')$, there exist $\beta_{1,j},\alpha_{1,j} >0 $ such that
\begin{equation}\label{ass:low_bdd_up_bdd}
\begin{gathered}
\alpha_{1,j}<\rho_{\theta,x,j}, \ \rho_{\theta,x',j} < \beta_{1,j}
\end{gathered}
\end{equation}
\end{theorem}
Then for any given $x$, we have:
\begin{equation}
\begin{aligned}
\mathbb{E}_{p(x'|x)}\left[\left|\ln \frac{q_\phi(y_x|x)}{q_\phi(y_x|x')}\right| \right] = \mathbb{E}_{p(x'|x)}\left[\left|\sum_{j=1}^l O\left((\mathbf{w}_j-\mathbf{w}_x)^T\left(h_\theta(x)-h_\theta(x')\right)\right) \right|\right]
\end{aligned}
\end{equation}
\begin{proof}
One proper way to prove is to show that 
\begin{equation}
    \left|\sum_{j=1}^l O\left((\mathbf{w}_j-\mathbf{w}_x)^T\left(h_\theta(x)-h_\theta(x')\right)\right) \right|
\end{equation} is an upper and lower bound of $\left|\ln \frac{q_\phi(y_x|x)}{q_\phi(y_x|x')}\right|$  for any given $x'$ and $x$ simultaneously. We firstly note that
\begin{equation}
    \ln \frac{q_\phi(y_x|x)}{q_\phi(y_x|x')} = \ln \rho_{\theta,x} - \ln \rho_{\theta,x'},
\end{equation}
to begin with, we replace $\alpha_{1,j}$ and $\beta_{1,j}$ with $\alpha^* = \underset{j}{\min} \{\alpha_{1,j}\}$ and $\beta^* = \underset{j}{\max}\{\beta_{1,j}\}$ in \eqref{ass:low_bdd_up_bdd},
the following holds because of the lipschitz continuousity of $\exp(x)$ with $x$ in a bound interval, as it is indicated by \ref{ass:low_bdd_up_bdd}: 
\begin{equation}
\begin{aligned}
    & \left|\exp\left((\mathbf{w}_j-\mathbf{w}_x)^Th_\theta(x)\right)-\exp\left((\mathbf{w}_j-\mathbf{w}_x)^Th_\theta(x')\right) \right| \\
 & \leq \beta_{1,j}\left|(\mathbf{w}_j-\mathbf{w}_x)^T\left(h_\theta(x)-h_\theta(x')\right)\right|,
\end{aligned}
\end{equation}
\begin{equation}
\begin{aligned}
    & \exp\left((\mathbf{w}_j-\mathbf{w}_x)^Th_\theta(x)\right)-\exp\left((\mathbf{w}_j-\mathbf{w}_x)^Th_\theta(x')\right)\\
 & \geq \alpha_{1,j}(\mathbf{w}_j-\mathbf{w}_x)^T\left(h_\theta(x)-h_\theta(x')\right)
\end{aligned}
\end{equation}
Then for $\rho_{\theta,x}-\rho_{\theta,x'} = \sum_{j=1}^l\exp\left((\mathbf{w}_j-\mathbf{w}_x)^Th_\theta(x)\right)-\exp\left((\mathbf{w}_j-\mathbf{w}_x)^Th_\theta(x')\right)$:
\begin{equation}
\begin{aligned}
&\left|\sum_{j=1}^l\exp\left((\mathbf{w}_j-\mathbf{w}_x)^Th_\theta(x)\right)-\exp\left((\mathbf{w}_j-\mathbf{w}_x)^Th_\theta(x')\right)\right|\\
&\leq \sum_{j=1}^l \beta_{1,j} \left|(\mathbf{w}_j-\mathbf{w}_x)^T\left(h_\theta(x)-h_\theta(x')\right)\right| \\
&\leq \beta^* \sum_{j=1}^l\left|(\mathbf{w}_j-\mathbf{w}_x)^T\left(h_\theta(x)-h_\theta(x')\right) \right|.
\end{aligned}
\end{equation}
Now for the first side, it is rather easier by directly applying Lagrange mean value theorem for $\ln x $ on the interval $\left[\min\{\rho_{\theta,x},\rho_{\theta,x'}\},\max\{\rho_{\theta,x},\rho_{\theta,x'}\}\right]$, then we have:
\begin{equation}
\begin{aligned}
|\ln \rho_{\theta,x}-\ln \rho_{\theta,x'}| \leq \frac{\left|\rho_{\theta,x} - \rho_{\theta,x'}\right|}{\alpha^*} = \left|\sum_{j=1}^l O\left((\mathbf{w}_j-\mathbf{w}_x)^T\left(h_\theta(x)-h_\theta(x')\right)\right) \right|.
\end{aligned}
\end{equation}
To prove the other side, we apply Lagrange mean value theorem for $e^x$ on the interval $\left[\min\{\ln \rho_{\theta,x},\ln \rho_{\theta,x'}\}, \max\{\ln \rho_{\theta,x},\ln \rho_{\theta,x'}\}\right]$, we have:
\begin{equation} \label{equation:lagrangeex}
\left|\rho_{\theta,x}-\rho_{\theta,x'}\right|=\left|e^{\ln \rho_{\theta,x}}-e^{\ln \rho_{\theta,x'}} \right|\leq \max \{\beta^*,1\} \left|\ln \rho_{\theta,x}-\ln \rho_{\theta,x'}\right|,
\end{equation}
then from \ref{equation:lagrangeex} we have 
\begin{equation}
    |\ln \rho_{\theta,x}- \ln \rho_{\theta,x'}| \geq \frac{\left|\rho_{\theta,x}-\rho_{\theta,x'}\right|}{\max \{\beta^*,1\}}\geq \frac{\rho_{\theta,x}-\rho_{\theta,x'}}{\max\{\beta^*,1\}},
\end{equation}
with the following taylor expansion is enough to finish the proof:
\begin{equation}
\begin{aligned}
&\exp\left((\mathbf{w}_j-\mathbf{w}_x)^Th_\theta(x)\right)-\exp\left((\mathbf{w}_j-\mathbf{w}_x)^Th_\theta(x')\right) \\
&= \exp\left((\mathbf{w}_j-\mathbf{w}_x)^Th_\theta(x')\right)\sum_{n=1}^\infty \frac{\left((\mathbf{w}_j-\mathbf{w}_x)^T\left(h_\theta(x)-h_\theta(x')\right)\right)^n}{n!}\\
&\geq \alpha^* O\left((\mathbf{w}_j-\mathbf{w}_x)^T\left(h_\theta(x)-h_\theta(x')\right)\right)
\end{aligned}
\end{equation}
\end{proof}

\section{Pseudo Code for Training Strategies}
\label{AppendixPseudoCode}
In the following algorithms, we denote the parameter of the model with $\phi$, the learning rate of the optimizer with $\eta$, minibatch of augmented samples with $\mathcal{D}_{\text{minibatch}}$, conditional distribution induced by the involved data augmentation with $p(x'|x)$, size of minibatch with $m$, size of training samples with $N$, flag value in training process with $t$, schedule of learning rate with $f_{\text{lr}}(\eta,t)$,  maximum steps of iteration $M_{\text{epoch}}$. In our experiment, we set $\lambda=0.5$ (We performed an ablation experiment on $\lambda$, please see Appendix \ref{choose_lambda}).

\begin{algorithm}[H]
    \caption{Standard Training Strategy}
    \label{algorithm: standardAlg}
    \SetKwData{Left}{left}\SetKwData{This}{this}\SetKwData{Up}{up}
    \SetKwFunction{Union}{Union}\SetKwFunction{FindCompress}{FindCompress}
    \SetKw{Init}{Init:}
    \Init{$\phi_0,\eta_0$, $\mathcal{D}_{\text{aug}} = \{\varnothing\}$, $\mathcal{D}_{\text{minibatch}} = \{\varnothing\}$, $t = 0$}
    \SetKwData{Left}{left}\SetKwData{This}{this}\SetKwData{Up}{up}
    \SetKwFunction{Union}{Union}\SetKwFunction{FindCompress}{FindCompress}
    \SetKwInOut{Input}{Input}\SetKwInOut{Output}{output}
    \Input{$\mathcal{D}_{\text{tr}}=\{(x_i,y_i)\}_{i=1}^N$, $f_{\text{lr}}(\eta,t)$, $m$, $N$,  $M_{\text{epoch}}$ }
    \For{epoch in $M_{\text{epoch}}$:}
        {\For{$\mathcal{D}_{\text{minibatch}} = \{(x_j,y_j)\}_{j=1}^m$ in $\mathcal{D}_{\text{tr}}$}
            { \For{$x_j$ in $\mathcal{D}_{\text{minibatch}}$}
                {
                    sample $x_j' \sim p(x'|x_j)$
                }
                $\mathcal{D}_{\text{aug}} = \{(x_j',y_j)\}_{j=1}^m$\\
             $\text{Loss}=-\frac{1}{N} \sum_{j=1}^{N}y_j \log q_\phi(y|x_j')$\\
             $\phi_{t+1}=\phi_{t}+\eta_t \cdot \nabla_{\phi}\text{Loss}$\\
            update learning rate $\eta_{t+1} = f_{\text{lr}}(\eta_t,t)$\\
            $t = t+1$
            }
        }
    
\end{algorithm}

\begin{algorithm}[H]
    \caption{Our Training Strategy}
    \label{algorithm: OurAlg}
    \SetKwData{Left}{left}\SetKwData{This}{this}\SetKwData{Up}{up}
    \SetKwFunction{Union}{Union}\SetKwFunction{FindCompress}{FindCompress}
    \SetKw{Init}{Init:}
    \Init{$\phi_0,\eta_0$, $\mathcal{D}_{\text{aug}} = \{\varnothing\}$, $\mathcal{D}_{\text{minibatch}} = \{\varnothing\}$, $t = 0$}
    \SetKwData{Left}{left}\SetKwData{This}{this}\SetKwData{Up}{up}
    \SetKwFunction{Union}{Union}\SetKwFunction{FindCompress}{FindCompress}
    \SetKwInOut{Input}{Input}\SetKwInOut{Output}{output}
    \Input{$\mathcal{D}_{\text{tr}}=\{(x_i,y_i)\}_{i=1}^N$, $f_{\text{lr}}(\eta,t)$, $m$, $N$,  $M_{\text{epoch}}$ }
    \For{epoch in $M_{\text{epoch}}$:}
        {\For{$\mathcal{D}_{\text{minibatch}} = \{(x_j,y_j)\}_{j=1}^m$ in $\mathcal{D}_{\text{tr}}$}
            { \For{$x_j$ in $\mathcal{D}_{\text{minibatch}}$}
                {
                    sample $x_j' \sim p(x'|x_j)$
                }
                $\mathcal{D}_{\text{aug}} = \{(x_j',y_j)\}_{j=1}^m$\\
            $\text{Loss}=-\frac{1}{N}\sum_{j=1}^{N} y_j[\log q_\phi(y|x_j)+\lambda(\log\frac{q_\phi(y|x_j)}{q_\phi(y|x_j')})]$\\
            $\phi_{t+1}=\phi_{t}+\eta_t \cdot\nabla_{\phi}\text{Loss}$\\
            update learning rate $\eta_{t+1} = f_{\text{lr}}(\eta_t,t)$\\
            $t = t+1$
        }
    }
\end{algorithm}

\section{Further discussion with other existing work}

\subsection{Comparsion with \cite{chen2020group}}\label{sec:diff_group}

The proposed mathematical framework in this paper differs from \cite{chen2020group} in three main aspects.

Firstly, \cite{chen2020group} defines the set of augmentations as a group $G$, and defines the augmentation operator on dataspace $X$ as group actions $G\times X\rightarrow X$. The scope of application of this definition is relatively narrow. In our work, the operator is defined as a measurable mapping $A(\theta, x)$. The set of operators is denoted as $\{A_i(\theta_i,\cdot)|\theta_i\in \Theta_i\}$ where $\Theta_i$ is the parameter space for data augmentation operator $A_i$, e.g., in the case of rotation, it can be $[0,2\pi]^n$. To support the claim we made, we give the following continuous differentiable color transformation operator (It was used in \cite{tian2021continuous}) as an example that does not meet the definition of group action. Color Adjustment is defined as a transformation in the spatial domain that is equally applied to each pixel. Let $n$ be the image size. For every coordinate $(x, y)$ with pixel vector $I_{x y}=\left[\begin{array}{lll}h_{x y} & \mathbf{s}_{x y} & \mathbf{v}_{x y}\end{array}\right]^{\top}$. The augmentation is then defined as:
$$
I_{x y}^{\prime}=\left[\begin{array}{c}
\mathrm{h}_{x y}^{\prime} \\
\mathrm{s}_{x y}^{\prime} \\
\mathrm{v}_{x y}^{\prime}
\end{array}\right]=\left[\begin{array}{l}
\alpha_{\mathrm{h}}+\left(1+\beta_{\mathrm{h}}\right)\left(\mathrm{h}_{x y}\right)^{\gamma_{\mathrm{h}}} \\
\alpha_{\mathrm{s}}+\left(1+\beta_{\mathrm{s}}\right)\left(\mathrm{s}_{x y}\right)^{\gamma_{\mathrm{s}}} \\
\alpha_{\mathrm{v}}+\left(1+\beta_{\mathrm{v}}\right)\left(\mathrm{v}_{x y}\right)^{\gamma_{\mathrm{v}}}
\end{array}\right]=(g,I_{xy}),
$$

Clearly, there does not exist a $g_1$ such that:
$$
(g_1,I_{xy})=\left[\begin{array}{l}
\alpha_{\mathrm{h}}^{\prime\prime} +\left(1+\beta_{\mathrm{h}}^{\prime\prime} \right)\left(\mathrm{h}_{x y}\right)^{\gamma_{\mathrm{h}}^{\prime\prime} } \\
\alpha_{\mathrm{s}}^{\prime\prime} +\left(1+\beta_{\mathrm{s}}^{\prime\prime} \right)\left(\mathrm{s}_{x y}\right)^{\gamma_{\mathrm{s}}^{\prime\prime} } \\
\alpha_{\mathrm{v}}^{\prime\prime} +\left(1+\beta_{\mathrm{v}}^{\prime\prime} \right)\left(\mathrm{v}_{x y}\right)^{\gamma_{\mathrm{v}}^{\prime\prime} }
\end{array}\right]=\left(g^{\prime},(g,I_{xy})\right).
$$
This means that such $g$ defined by $\boldsymbol{\alpha}_{\mathrm{CA}}=\left[\begin{array}{lllllllll}
\alpha_{\mathrm{h}} & \beta_{\mathrm{h}} & \gamma_{\mathrm{h}} & \alpha_{\mathrm{s}} & \beta_{\mathrm{s}} & \gamma_{\mathrm{s}} & \alpha_{\mathrm{v}} & \beta_{\mathrm{v}} & \gamma_{\mathrm{v}}
\end{array}\right]^{\top}$ cannot be a group action in sample space $X$. In conclusion, our definition of augmentation operator can be applied to a wider range.

Secondly, the proposed framework in this paper pays more attention to transforming the probability measure on the parameter space $\Theta$ to the probability measure $p(x'|x)$ on the sample space $X$, while this hasn't been considered in work \cite{chen2020group}. In this work, we do not directly define the data distribution of the augmented data on the group orbit. Instead, we define the augmented samples for every clean sample $x$ by considering the measurable map $A(\cdot,x)$. This is inspired by the insight that one generates augmented samples through sampling from the parameter space.

Thirdly, in the aforementioned paper, they proved that the shifted population risk and the original population risk are asymptotically approximate or non-asymptotically approximate (under different assumptions). While one of the main premises of our work is that the shifted population risk and the original population risk are not the same. There is a gap between them, and this gap is the key to the  generalization of the model as it is shown in the ablation experiment in Appendix \ref{choose_lambda}, indicating that adding coefficient $\lambda$ to weaken its impact on major features helps improving the generalization ability. This gap, however, is not paid attention to in \cite{chen2020group}.

\subsection{Comparsion with \cite{16he2019data}}\label{sec:diff_he}
Our work differs from \cite{16he2019data} in the following aspects:

Firstly, although we mention similar tools like VC dimension and information gains in our analysis, our main focus is on the decomposition of the loss function. We specifically investigate the impact of the consistency regularization term and prove through a theorem that it can affect the learning of major features. Equation \eqref{eq:key_equality} demonstrates that this impact can worsen the generalization performance of the model, as $\hat{R}_{f_{\mathrm{aug}}}(p)$ is part of the upper bound of the generalization error, and the consistency regularization term can lead to different weights for some major features as it is indicated by \eqref{theorem: OrderofSecondterm}, particularly in the early training stage. Our proposed $\lambda$ helps to reduce such inconsistency.

Secondly, \cite{16he2019data} does not realize the effect of the consistency regularization term. They propose a method called refined data augmentation, which involves refining the models without intensive data augmentation at the end of the training stage. This approach differs from ours. Moreover, it is worth noting that the mentioned decomposition cannot occur without our proposed mathematical framework, which holds great importance.

\section{Validation Experiment of Selecting $\lambda$}
\label{choose_lambda}
The choice of $\lambda$ depends on the involved data augmentations since the value of $\|h_d(x)-h_d(x')\|$ depends on them, the $\lambda$ should not be too small since a small $\lambda$ directly ignores the second term and the training strategy degenerate to the empirical original population risk.
An empirical choice of $\lambda$ is 0.5, which we get from a series of validation experiments conducted on CIFAR100 with Resnet-18. We have conducted three independent experiments and calculated the mean values as the results, which are shown in the following Table \ref{lambda_result}.

\begin{table}[htbp]
  \centering
  \caption{Error rates (\%) of ResNet-18 on the test set of CIFAR100 with different $\lambda$.}
    \resizebox{\linewidth}{!}{\begin{tabular}{cccccc}
    \toprule
    Standard (1bbs) & $\lambda=0.0001$ & $\lambda=0.1$   & $\lambda=0.2$   & $\lambda=0.3$   & $\lambda=0.4$ \\
    \midrule
      $21.21\pm0.18$  & $35.05\pm0.10$* & $22.84\pm0.07$* & $21.39\pm0.22$* & $20.85\pm	0.16$ & $20.71\pm0.40$ \\
    \midrule
      & $\lambda=0.5$   & $\lambda=0.6$   & $\lambda=0.7$   & $\lambda=0.8$   & $\lambda=0.9$   \\
    \midrule
    &\textbf{$20.22\pm0.05$} & $20.59\pm0.12$ & $20.52\pm0.26$ & $20.48\pm0.34$ & $20.86\pm0.11$  \\
    \bottomrule
    \multicolumn{6}{l}{*Results that are below the standard performance.}
    \end{tabular}}
  \label{lambda_result}
\end{table}
In this table, we see that a small $\lambda$ $(\lambda = 0.0001, 0.1, 0.2)$ leads to a worse performance than the standard method for the occurrence of degeneration. And when $\lambda = 0.5$, we achieve the best result among these $\lambda$.

\section{Experiment Implementation}
\subsection{Datasets Splitting Details}
\label{DSD}
CIFAR-10/100 consist of 60,000 32$\times$32 color images in 10 and 100 classes, respectively. Food-101 consist of 101,000 color images in 101 classes. ImageNet (ILSVRC2012) \cite{russakovsky2015imagenet} consists of approximately 1.33 million $224\times224$ images in 1000 classes. For each of them, a validation set is split from the training set to find networks with the best performances, which is shown in Table \ref{split}.

\begin{table}[htbp]
  \centering
  \caption{Dataset Partitioning}
  \resizebox{\textwidth}{!}{
    \begin{tabular}{cccccc}
    \toprule
    Dataset & Train Dataset & Test Dataset & Valid Dataset & Classes & Evaluation Criterion \\
    \midrule
    ImageNet &  1281167  & 0 &  50000  &  1000  & Top-1 and Top-5 accuracy \\
    Food101 & 68175 & 25250 & 7575  & 101   & Top-1 accuracy \\
    CIFAR10 & 45000 & 10000 & 5000  & 10    & Top-1 accuracy \\
    CIFAR100 & 45000 & 10000 & 5000  & 100   & Top-1 accuracy \\
    \bottomrule
    \end{tabular}}
  \label{split}%
\end{table}%

As for the OOD Scenario experiments, there's no need to split datasets, since we employed the leave-one-domain-out strategy for OOD validation.

\subsection{Data augmentation}
\label{D_aug}
In the setting of data augmentation, horizontal flips with p = 0.5, random crops from image padded by 1/8 pixels of the original size on each side (4 pixels for 32$\times$32 and 28 pixels for 224$\times$224) and random rotation with degrees = 15. For ResNet on ImageNet and WideResNet, mean/std normalization is added. Since the standard ResNet is designed for ImageNet, the data are resized to 224$\times$224 when we use ResNet-18/50, which is not applied to WideResNet. All images in baseline experiments are processed with the above augmentation. However, for our strategy, we will keep both the augmented data and the data that have only been resized or normalized.

As for the OOD Scenario experiments, we use the same augmentation as Domainbed's \cite{gulrajani2020search} both in ERM algorithm and ours, which is composed of a random resized crop, a random horizontal flip, a color jitter and a random gray scale. Moreover, all images are normalized before training.

\subsection{Training Details}
\label{T_Detail}
All models are trained on a single GPU (NVIDIA RTX A6000 or NVIDIA A40 for ImageNet and NVIDIA GeForce RTX 3090 for others) using SGD with a weight decay of $5\times10^{-4}$ and a momentum of 0.9 (Nesterov momentum for WideResNet) for 200 epochs. For CIFAR-10/100 and Food101, the basic batch size (bbs) is set to 100 for ResNet and 128 for WideResNet and the basic learning rate (blr) is set to 0.1. In addition, warmup is used for 5 epochs for ResNet-18 and 10 epochs for others with a cosine learning rate schedule. For ImageNet, we use the bbs of 256 for ResNet-50 and 192 for ResNet-101. The basic learning rate (blr) starts at 0.1 and is divided by 10 after 60, 120, 130 and 180 epochs with the warmup for 10 epochs. To accelerate the speed, we adopt mixed precision training with torch.cuda.amp on ImageNet.

As for the OOD Scenario experiments, all models are trained on a single NVIDIA GeForce 3090 using AdamW with no weight decay (following \cite{gulrajani2020search}). The learning rate is set to 0.001 with a cosine annealing scheduler. Considering the experimental setup in \cite{gulrajani2020search}, we set a batchsize of 40 and a training epoch of 50 for every test domain.

As for the long-tail imbalance, all models are trained on a single NVIDIA GeForce 3090 using AdamW with the same hyperparameters as the standard scenario experiment of CIFAR-10. For example, we use the learning rate of 0.1 and the batchsize of 100 with a cosine learning rate schedule. For our training strategy, the $\lambda$ is set to 0.5. 

\subsection{Results with Error Bar}
\label{error_bar}
Due to limited space, we did not include error bars with data in the main text. Instead, they are presented here. The following Table \ref{result_with_error_bar} to Table \ref{tab:LT} are the data of our experiment results.

\begin{table}[htbp]
  \centering
  \caption{Top-1 accuracy (\%) of CIFAR10/100 and Food101 on the test set}
    \begin{tabular}{lccc}
    \toprule
    Model & CIFAR10 & CIFAR100 & Food101 \\
    \midrule
    ResNet-18 &       &       &  \\
    Standard (1$\times$ bbs) &  $95.59\pm 0.09$  & $78.79\pm0.18$ & $78.00\pm0.10$ \\
    Standard (2$\times$ bbs) & $95.48\pm 0.10$ & $78.24\pm 0.11$ & $77.20\pm 0.18$ \\
    Ours (1$\times$ bbs) & \textbf{$96.22\pm0.07$} & \textbf{$79.78\pm0.05$} & \textbf{$78.26\pm0.03$} \\
    \midrule
    ResNet-50 &       &       &  \\
    Standard (1$\times$ bbs) & $95.53\pm0.11$ & $80.34\pm0.14$ & $82.55\pm0.19$ \\
    Standard (2$\times$ bbs) & $95.27\pm0.07$	& $80.16\pm0.25$ & $82.29\pm0.02$ \\
    Ours (1$\times$ bbs) & \textbf{$96.06\pm0.04$} & \textbf{$81.90\pm0.17$} & \textbf{$83.46\pm0.09$} \\
    \midrule
    WideResNet-40-2 &       &       &  \\
    Standard (1$\times$ bbs) & $94.78\pm0.08$	& $74.23\pm0.05$ & - \\
    Standard (2$\times$ bbs) & $94.55\pm0.06$	& $74.11\pm0.32$ & - \\
    Ours (1$\times$ bbs) & \textbf{$95.19\pm0.17$} & \textbf{$76.13\pm0.24$} & - \\
    \midrule
    WideResNet-28-10 &       &       &  \\
    Standard (1$\times$ bbs) & $95.91\pm0.11$	& $79.39\pm0.13$ & - \\
    Standard (2$\times$ bbs) & $95.70\pm0.04$	& $78.76\pm0.23$ & - \\
    Ours (1$\times$ bbs) & \textbf{$96.17\pm0.07$} & \textbf{$80.38\pm0.29$} & - \\
    \bottomrule
    \end{tabular}%
  \label{result_with_error_bar}%
\end{table}

\begin{table}[htbp]
  \centering
  \caption{Top-1 accuracy (\%) of ImageNet (ILSVRC2012) on the test set. Due to the limited computation resources, we are forced to adopt mixed precision training with torch.cuda.amp.}
    \begin{tabular}{lcc}
    \toprule
    Model & \multicolumn{1}{l}{Top-1} & \multicolumn{1}{l}{Top-5}  \\
    \midrule
    ResNet-50 &  &  \\
    Standard (1$\times$bbs) & 74.23 & 91.81\\
    Standard (2$\times$bbs) & 73.98 & 91.80\\
    Ours ($1\times$bbs) & 74.74 & 92.51\\
    \midrule
    ResNet-101 &  \\
    Standard (1$\times$bbs) & 75.72 & 92.90\\
    Standard (2$\times$bbs) & 75.43  & 92.82\\
    Ours (1$\times$bbs) & 75.92  & 92.93\\
    \bottomrule
    \end{tabular}%
  \label{imagenet_result}%
\end{table}%

\begin{table}[htbp]
  \centering
  \caption{Top-1 accuracy (\%) of PACS on the test domain}
    \begin{tabular}{ccccccc}
    \toprule
    \multicolumn{2}{c}{Method \textbackslash Test Domain} & Art Paint & Cartoon & Photo & Sketch & (Average) \\
    \midrule
    \multirow{2}[1]{*}{ResNet-18} & ERM   & $72.46\pm1.36$ & $73.52\pm0.57$ & $87.87\pm0.61$ & $72.43\pm0.75$ & $76.57\pm0.53$ \\
                                  & Ours  & $72.91\pm0.97$ & $80.00\pm0.50$ & $87.80\pm0.36$ & $80.75\pm0.18$ & $79.86\pm0.48$ \\
    \midrule
    \multirow{2}[1]{*}{ResNet-50} & ERM   & $75.39\pm0.27$ & $78.89\pm1.17$ & $92.05\pm0.59$ & $76.80\pm0.97$ & $80.78\pm0.31$ \\
                                  & Ours  & $77.26\pm0.16$ & $82.96\pm1.02$ & $92.25\pm0.29$ & $82.23\pm1.10$ & $83.67\pm0.10$ \\
    \bottomrule
    \end{tabular}%
  \label{tab:pacs}%
\end{table}%

\begin{table}[htbp]
  \centering
  \caption{Results of the long-tailed scenario experiment.}
    \begin{tabular}{ccllll}
    \toprule
    \multicolumn{2}{p{10em}}{Dataset Name} & \multicolumn{4}{c}{Long-Tailed CIFAR10} \\
    \midrule
    \multicolumn{2}{p{10em}}{Imbalance} & \multicolumn{1}{c}{100} & \multicolumn{1}{c}{50} & \multicolumn{1}{c}{20} & \multicolumn{1}{c}{10} \\
    \midrule
    \multirow{3}[0]{*}{Baseline} & ACC   & $74.91\pm 0.48$ & $81.55\pm 0.08$ & $87.56\pm 0.25$ & $90.47\pm 0.09$ \\
          & AUC   & $96.48\pm 0.03$ & $97.73\pm 0.01$ & $98.79\pm 0.02$ & $99.22\pm 0.00$ \\
          & AP    & $82.96\pm 0.18$ & $88.57\pm 0.13$ & $93.69\pm 0.12$ & $95.67\pm 0.04$ \\
    \midrule
    \multirow{3}[0]{*}{Ours} & ACC   & $75.23\pm 1.76$ & $82.37\pm 0.65$ & $88.11\pm 0.23$ & $90.94\pm 0.09$ \\
          & AUC   & $96.69\pm 0.12$ & $97.97\pm 0.11$ & $98.95\pm 0.01$ & $99.36\pm 0.00$ \\
          & AP    & $83.27\pm 0.98$ & $89.47\pm 0.28$ & $94.43\pm 0.10$ & $96.41\pm 0.03$ \\
    \bottomrule
    \end{tabular}%
  \label{tab:LT}%
\end{table}%

\subsection{Accuracy Curves}\label{appendix:accuracycurve}

\begin{figure}[h]
    \centering
    \subfloat[Resnet-50]
    {
        \label{Figure 3}
        \includegraphics[width=0.42\linewidth]{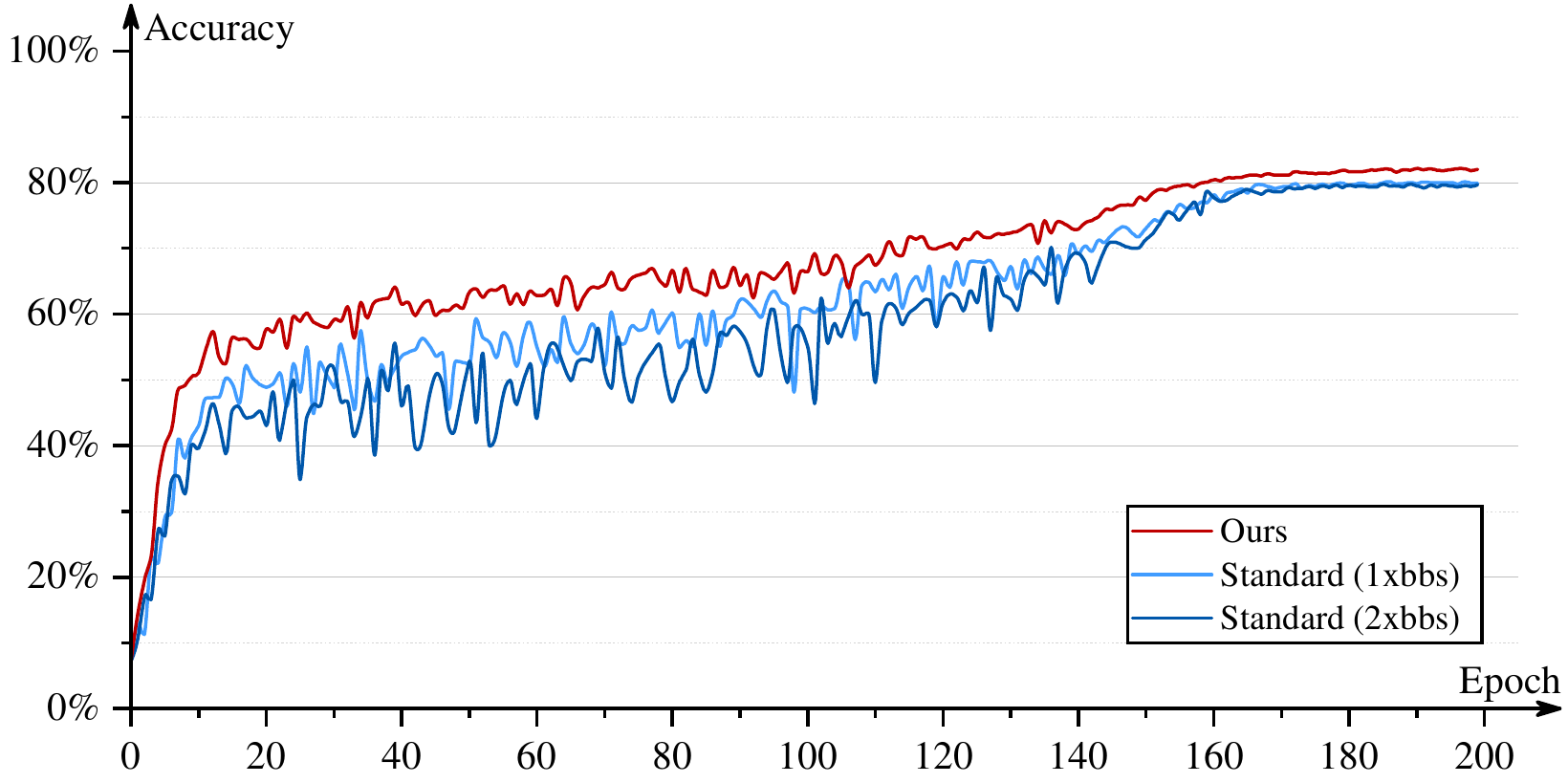}
    }
    \subfloat[WRN-40-2]
    {
        \label{Figure 4}
        \includegraphics[width=0.42\linewidth]{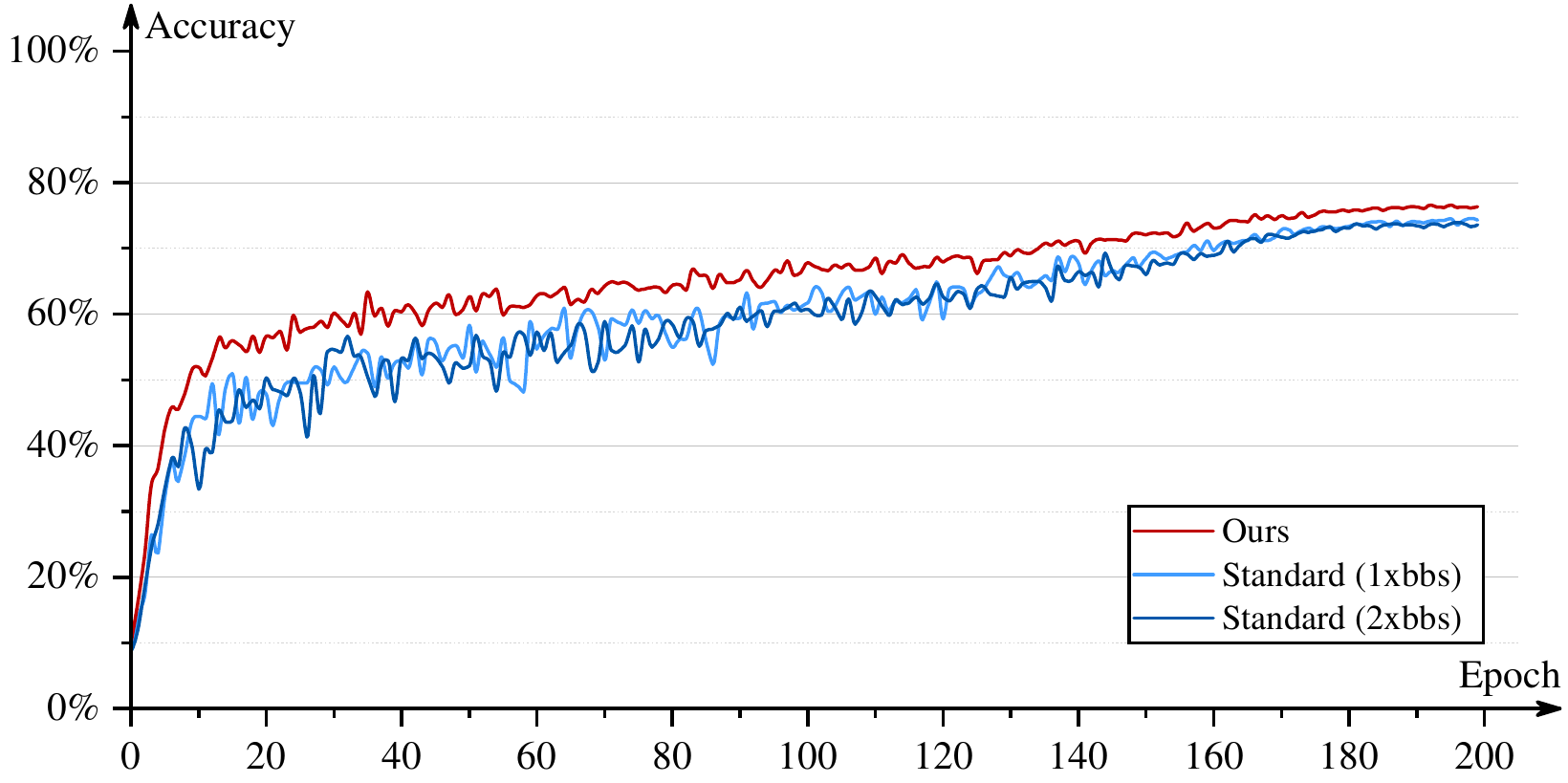}
    } 
    \caption{Top-1 accuracy  curve of Resnet-50 and WideResNet-40-2 training on CIFAR-100. The left is Resnet-50 and the right is WideResNet-40-2.}
    \label{Fcc}
\end{figure}

We randomly selected one set from the three experiments for plotting, which are shown in Figure \ref{xunliantuxiang}.

\begin{figure}[htbp]
    \centering
    \subfloat[ResNet-18 CIFAR10]
    {
        \includegraphics[width=0.42\linewidth]{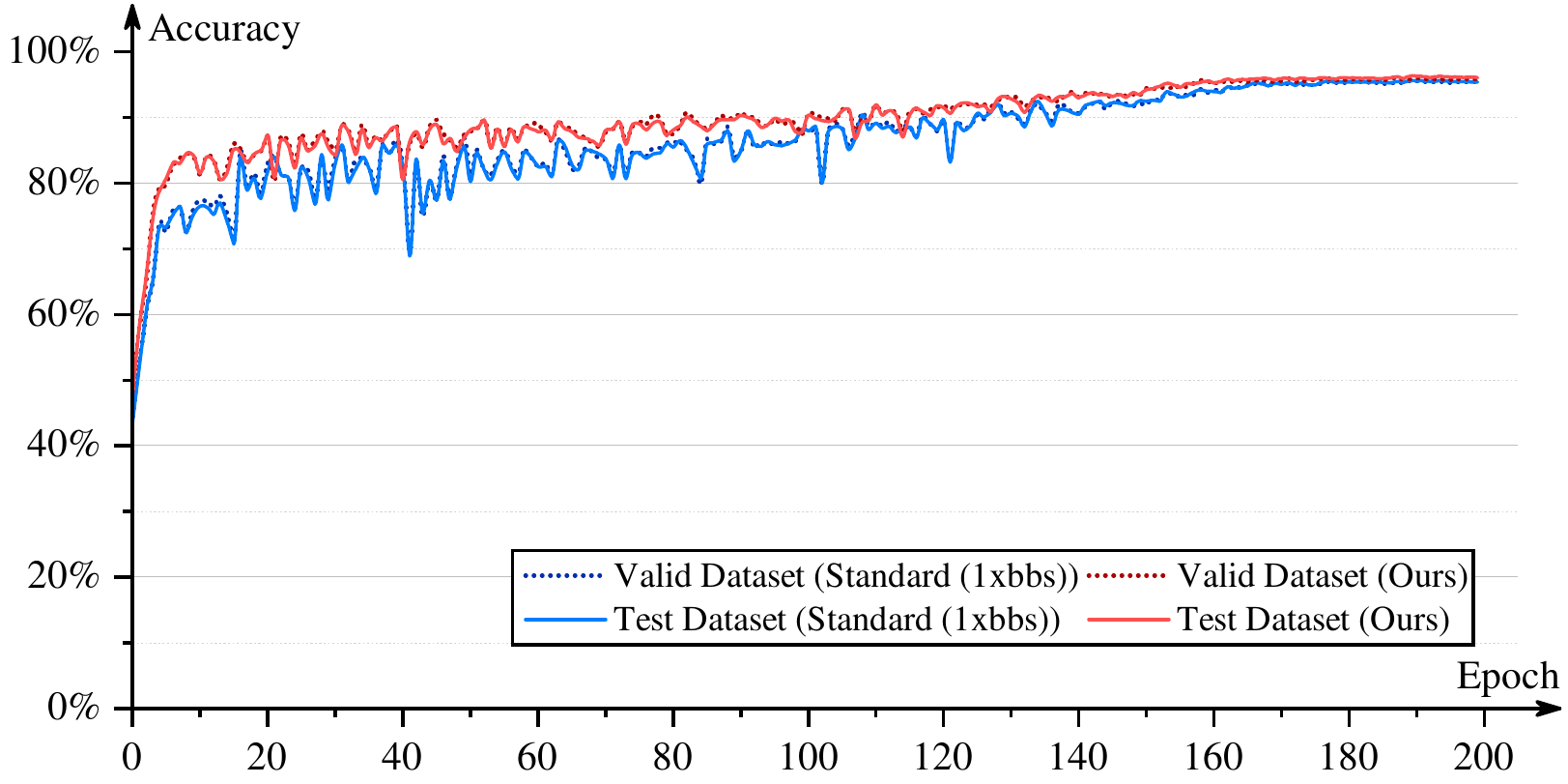}
    }
    \subfloat[ResNet-18 CIFAR100]
    {
        \includegraphics[width=0.42\linewidth]{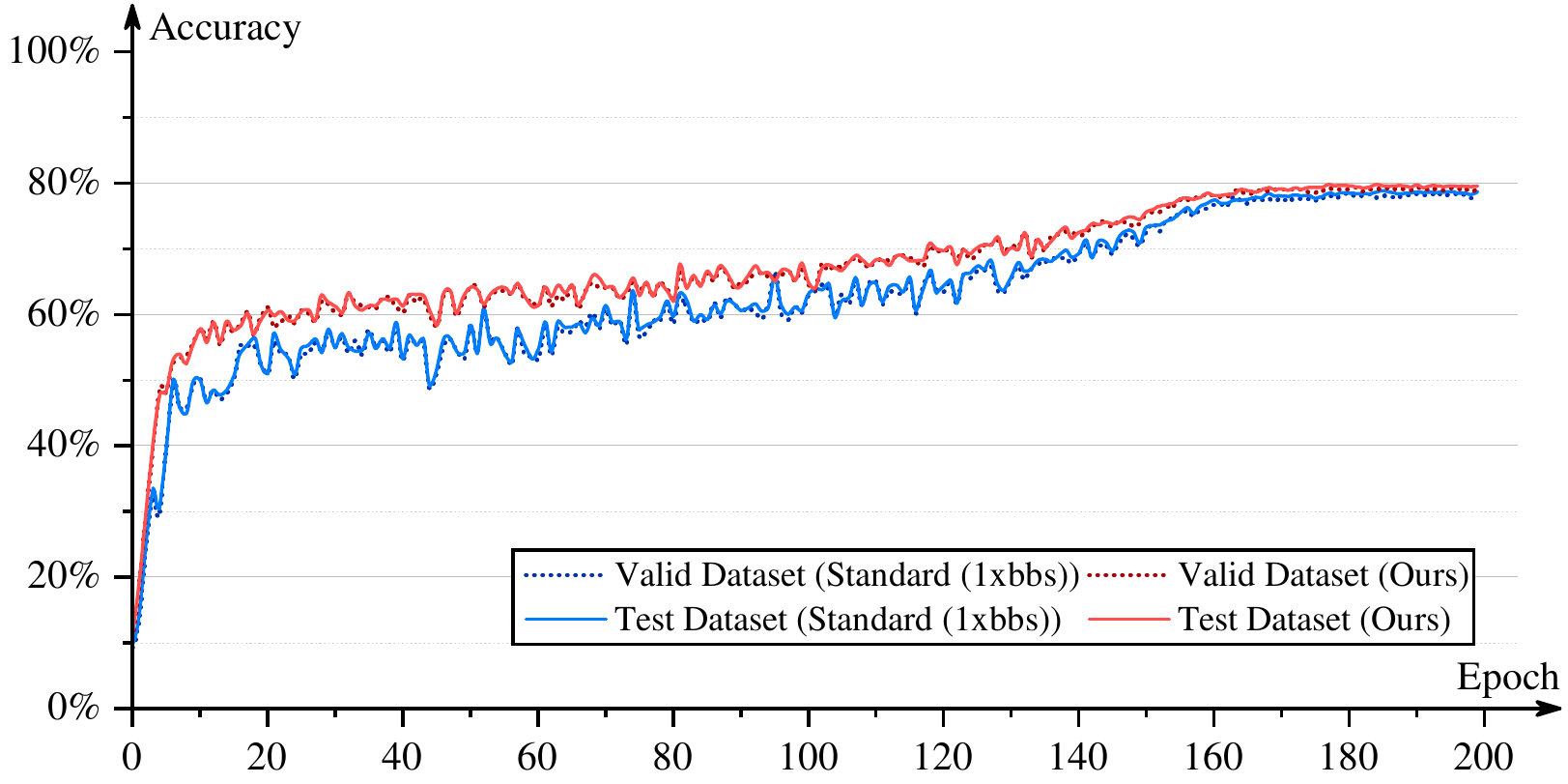}
    }

    \subfloat[ResNet-18 Food101]
    {
        \includegraphics[width=0.42\linewidth]{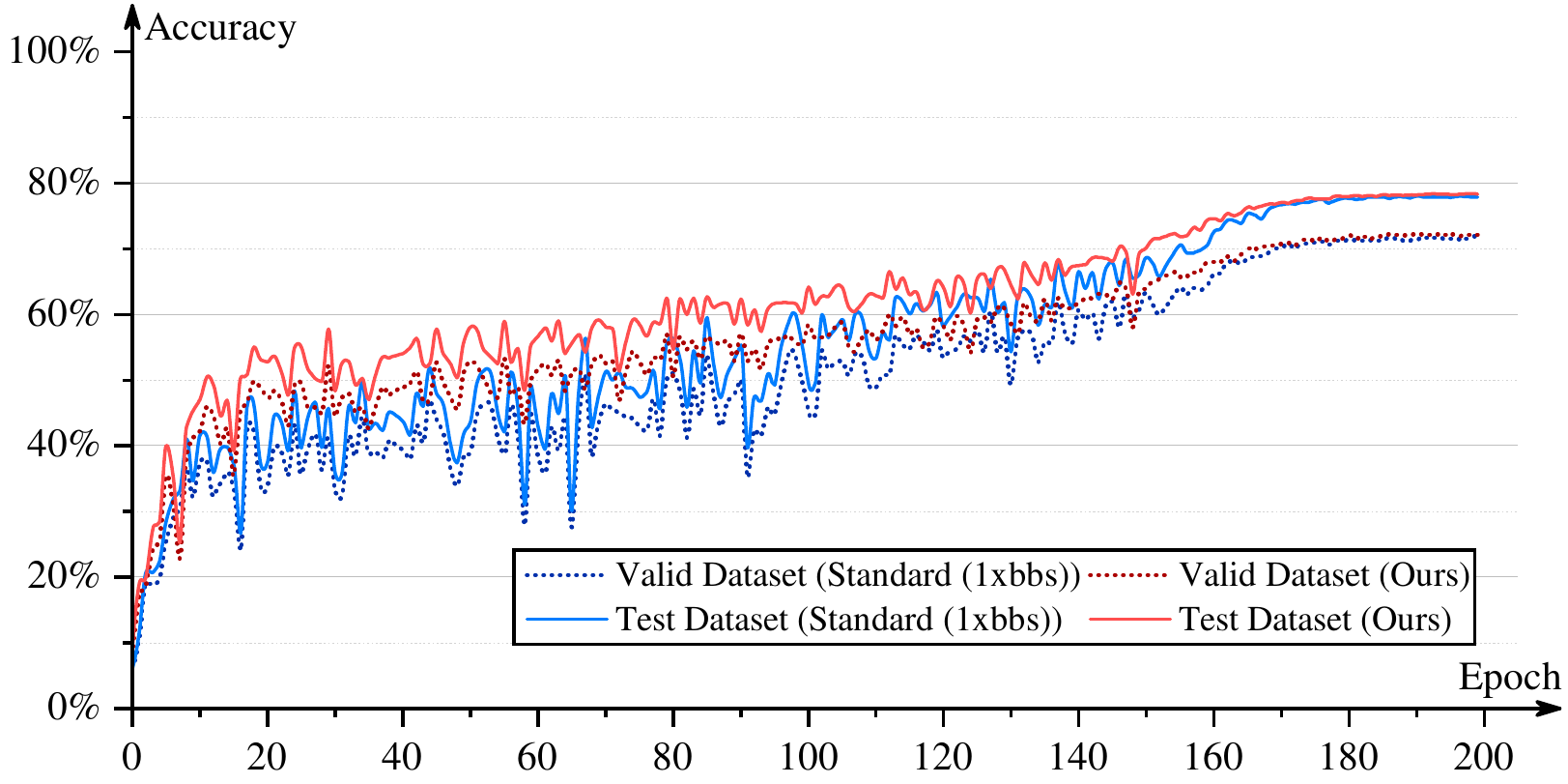}
    }
    \subfloat[ResNet-50 CIFAR10]
    {
        \includegraphics[width=0.42\linewidth]{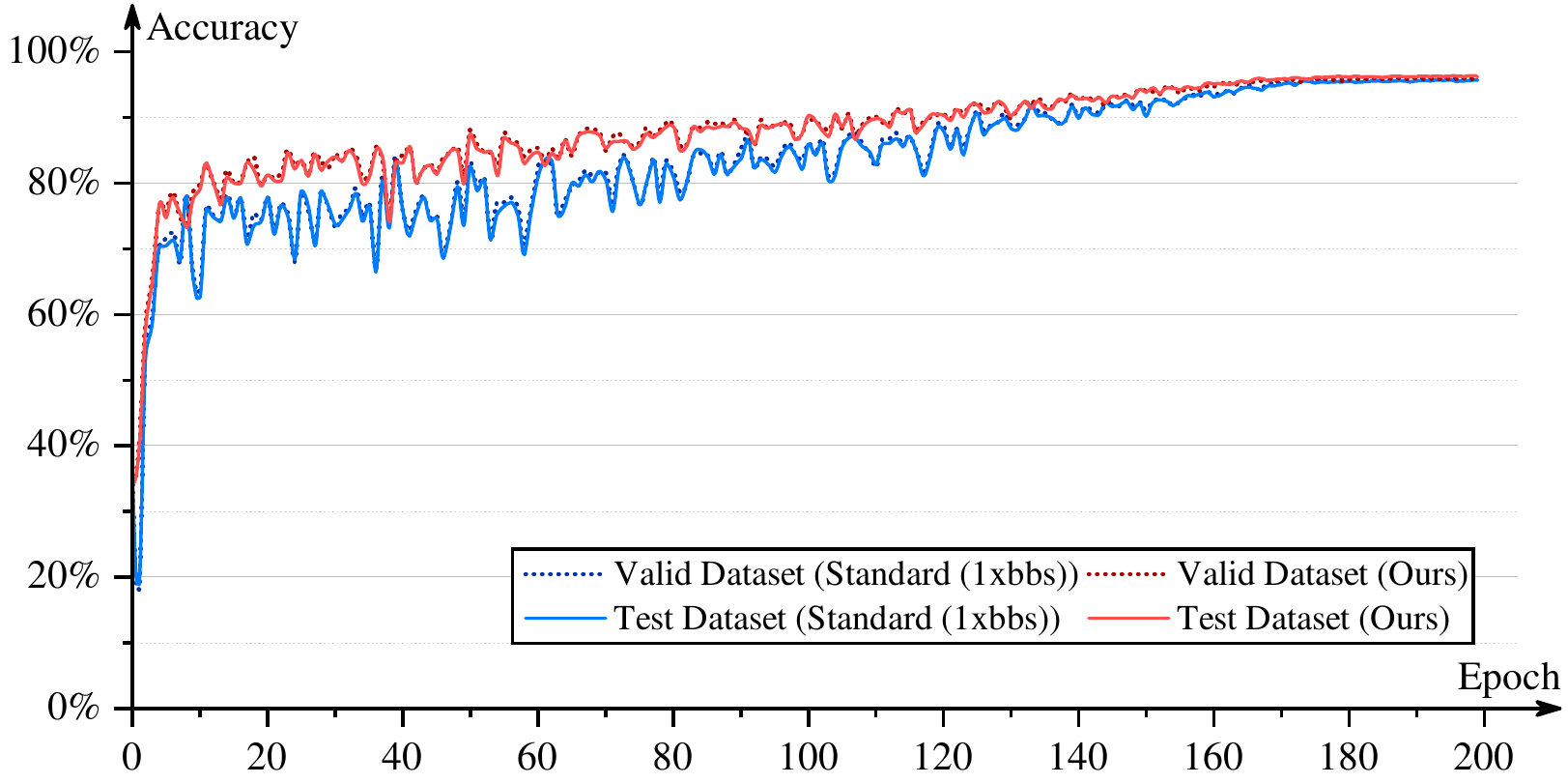}
    }
    \caption{Accuracy curves}
\end{figure}

\begin{figure}[htbp]
    \ContinuedFloat
    \centering
    \subfloat[ResNet-50 CIFAR100]
    {
        \includegraphics[width=0.42\linewidth]{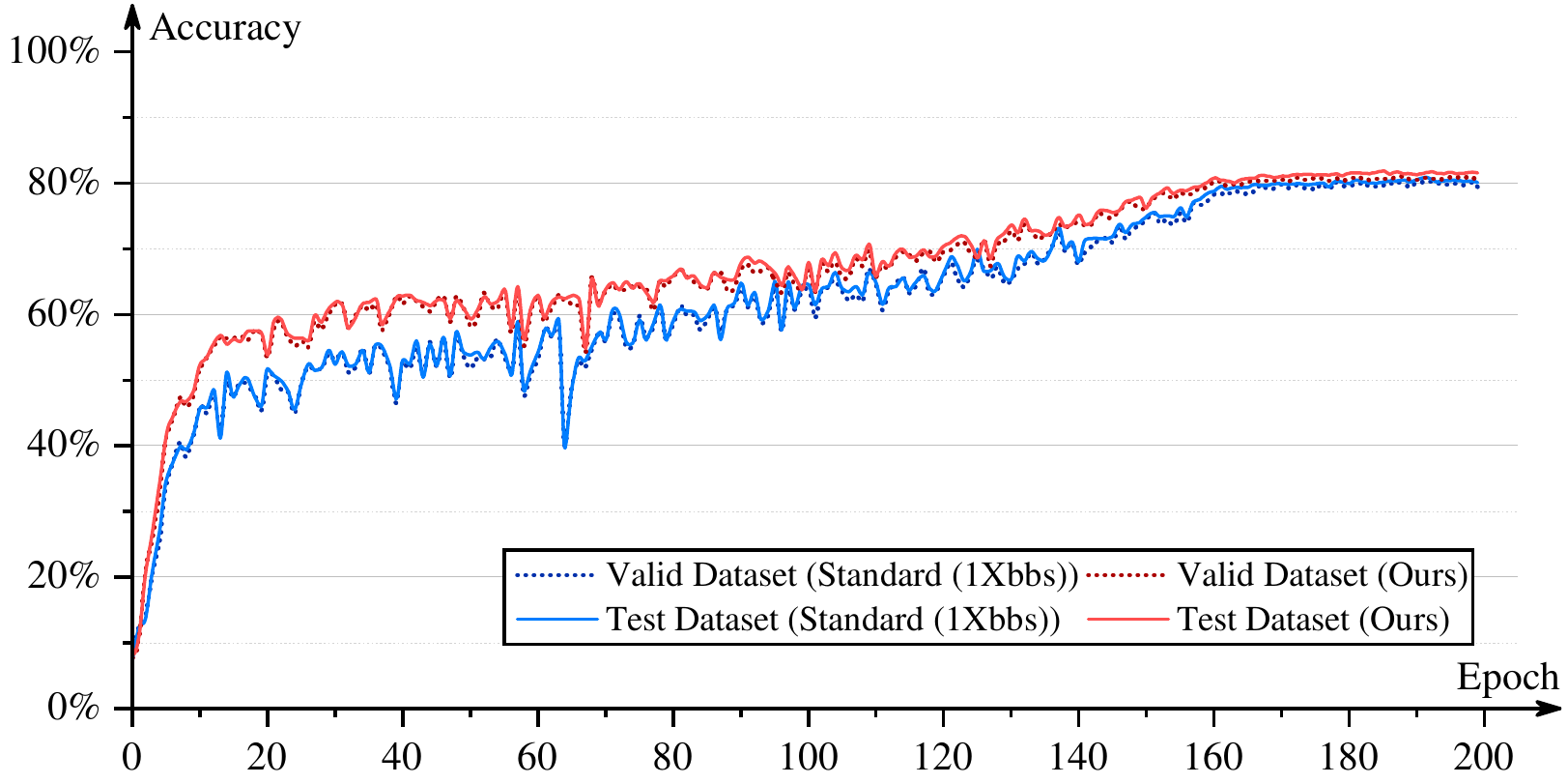}
    }
    \subfloat[ResNet-50 Food101]
    {
        \includegraphics[width=0.42\linewidth]{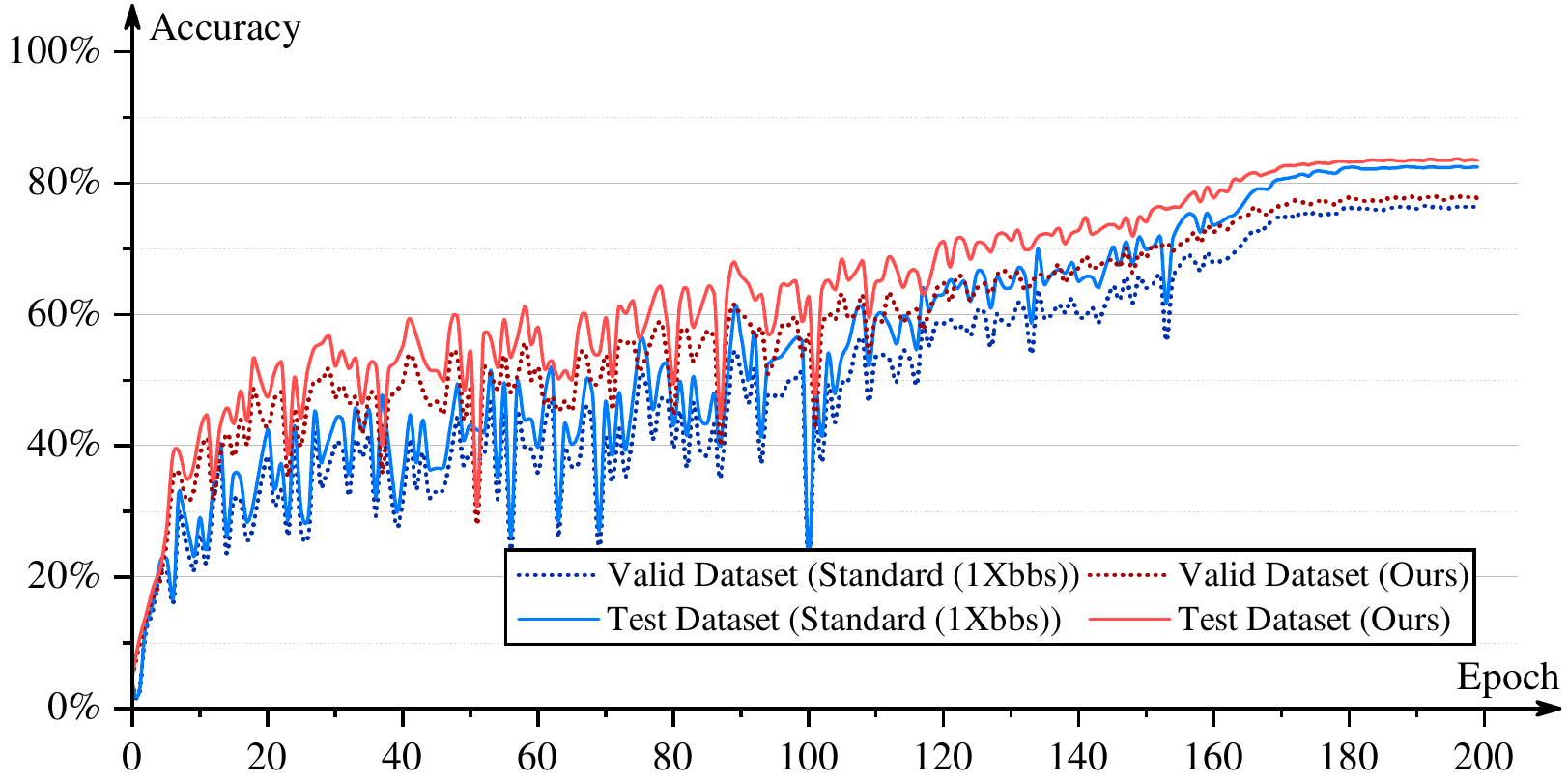}
    }
    
    \subfloat[WideResNet-40-2 CIFAR10]
    {
        \includegraphics[width=0.42\linewidth]{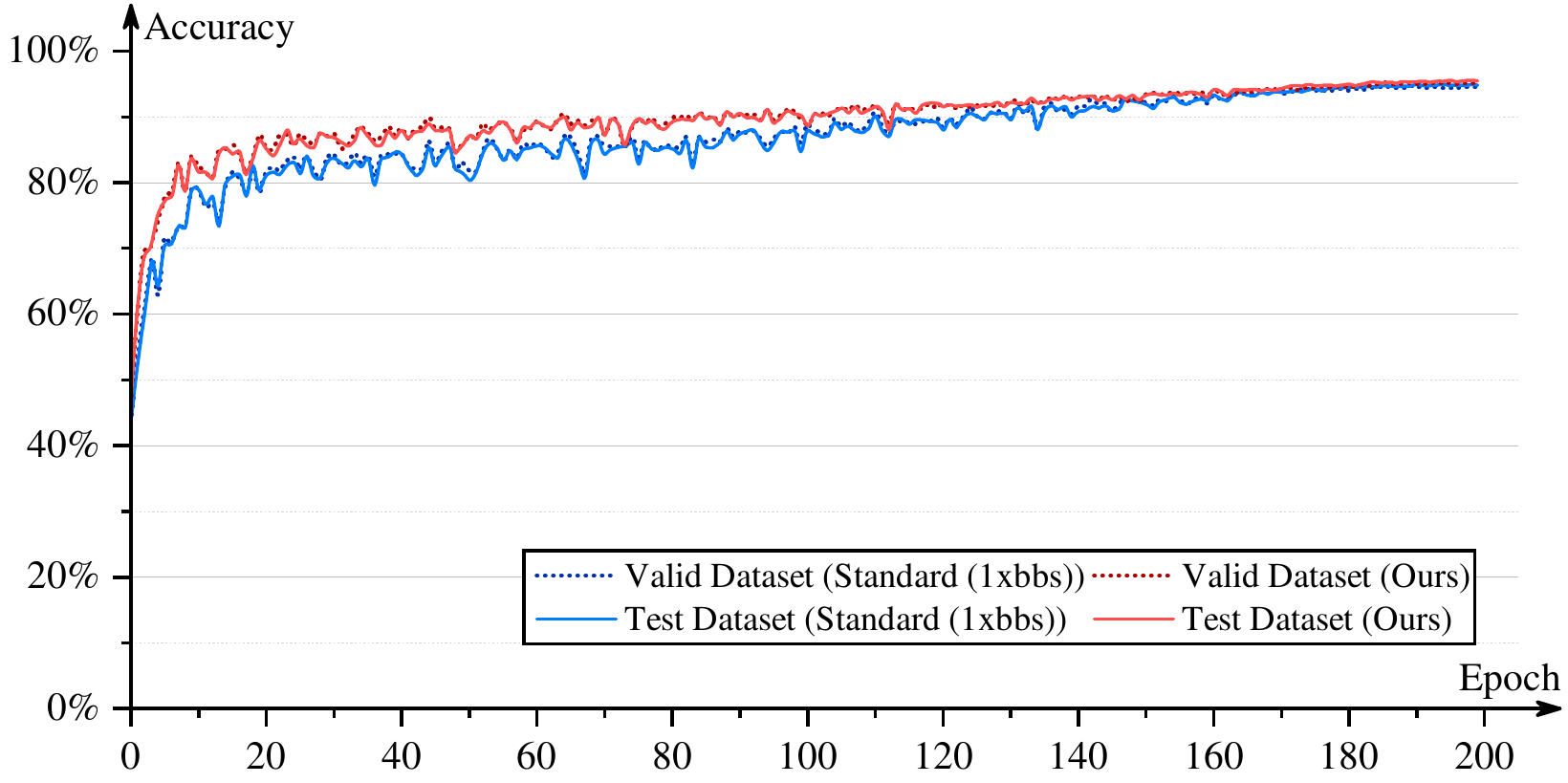}
    }
    \subfloat[WideResNet-40-2 CIFAR100]
    {
        \includegraphics[width=0.42\linewidth]{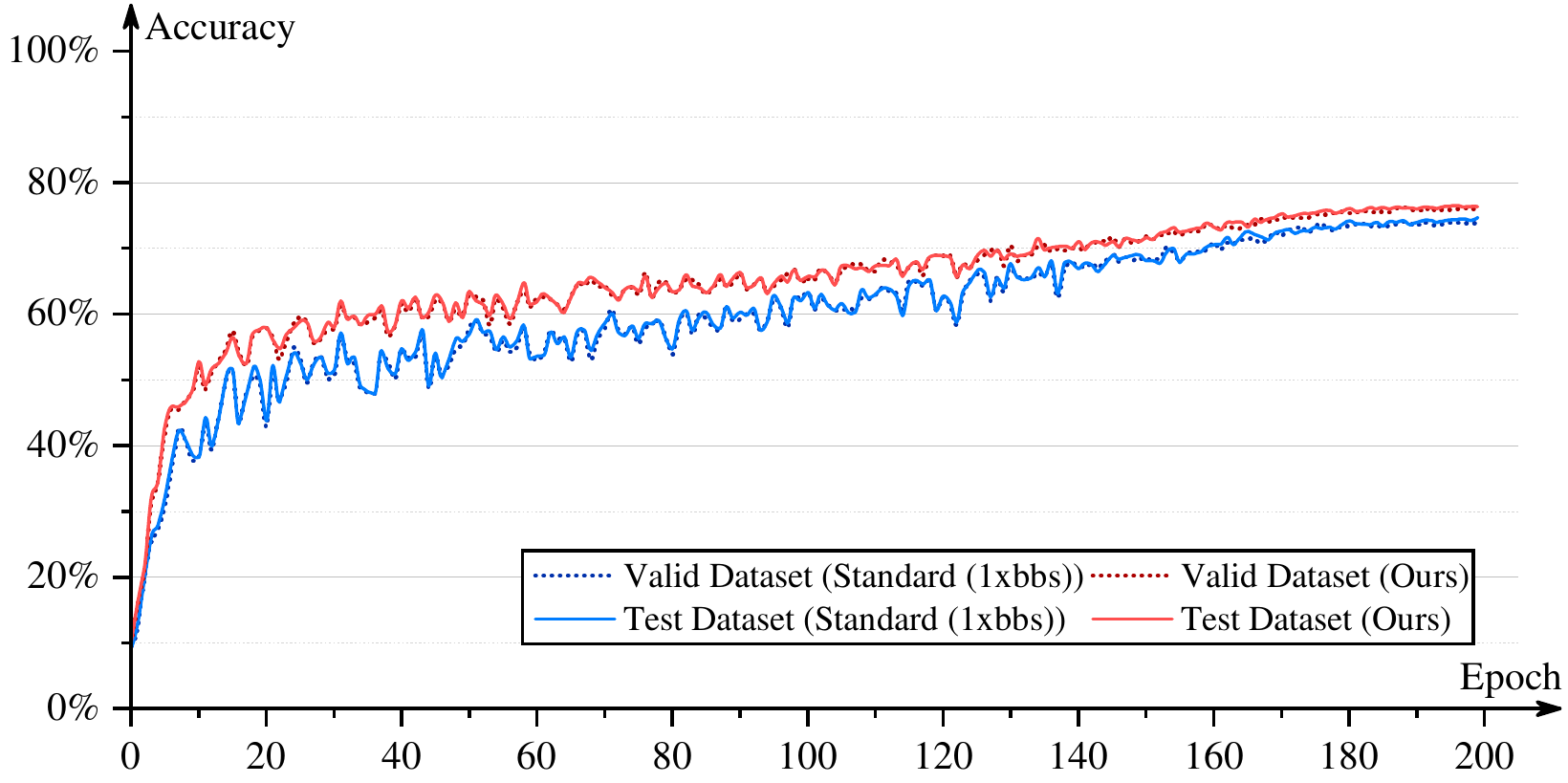}
    }
    
    \subfloat[WideResNet-28-10 CIFAR10]
    {
        \includegraphics[width=0.42\linewidth]{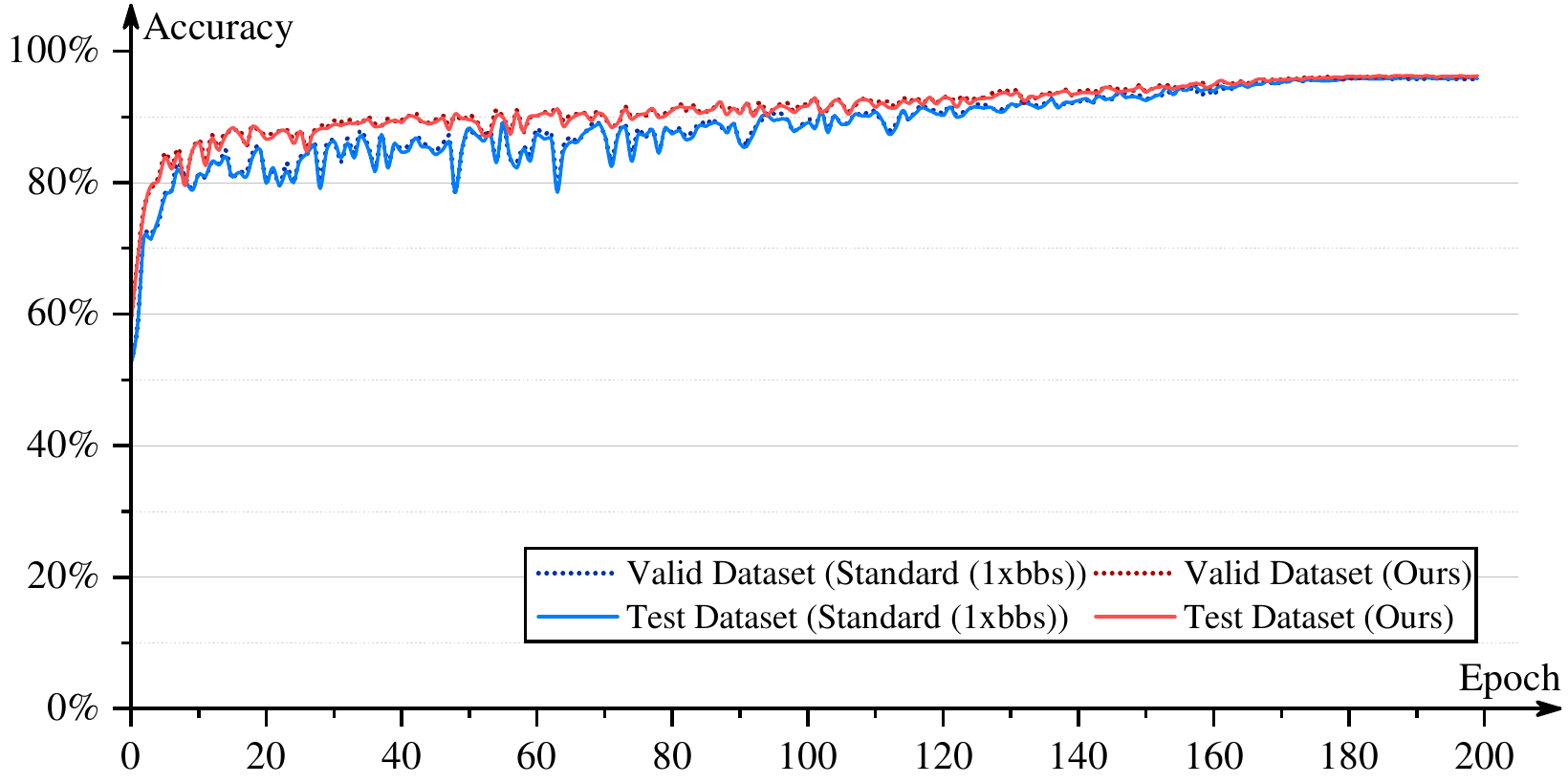}
    }
    \subfloat[WideResNet-28-10 CIFAR100]
    {
        \includegraphics[width=0.42\linewidth]{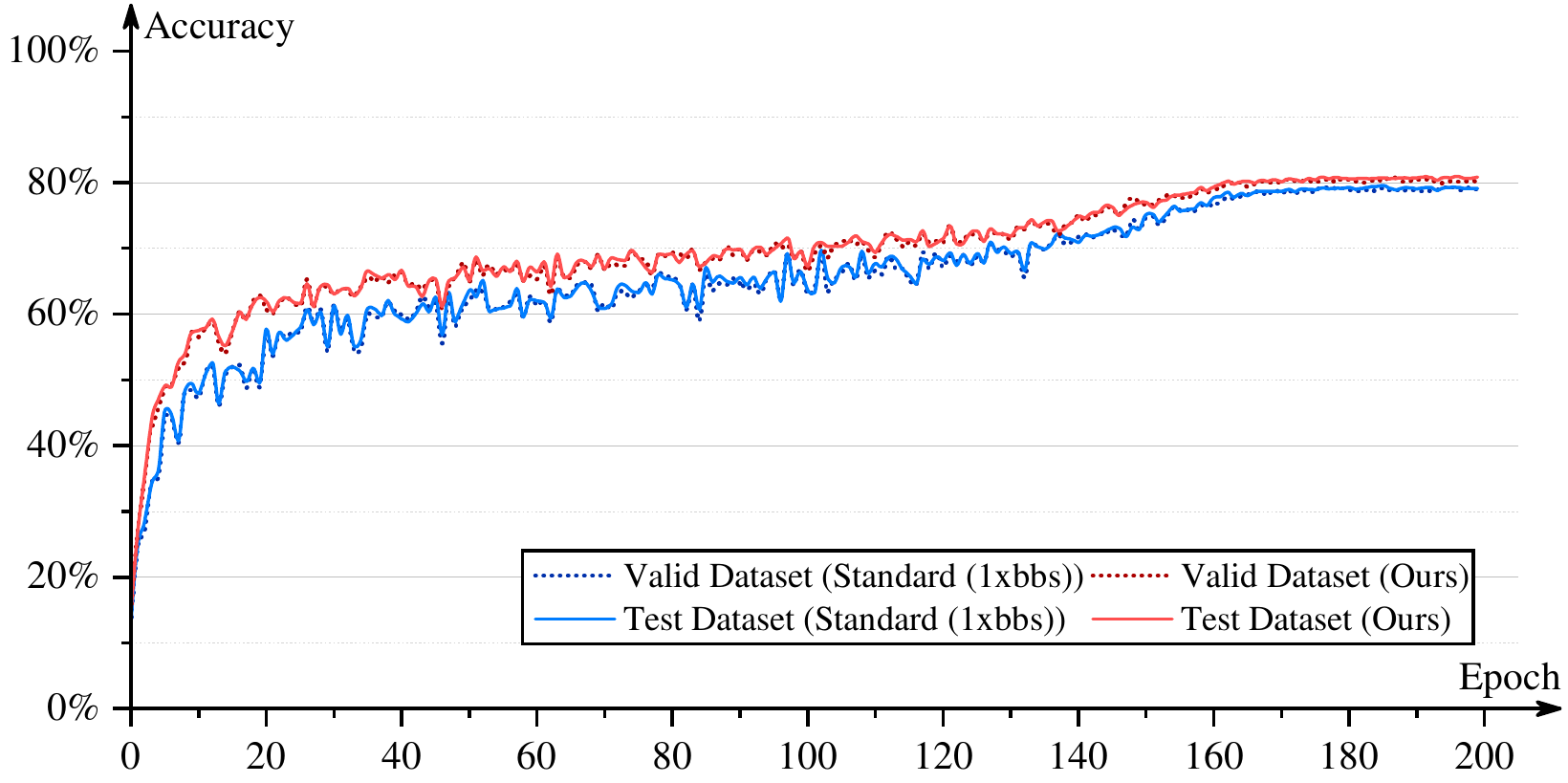}
    }

    \subfloat[ResNet-50 ImageNet top-1]
    {
        \includegraphics[width=0.42\linewidth]{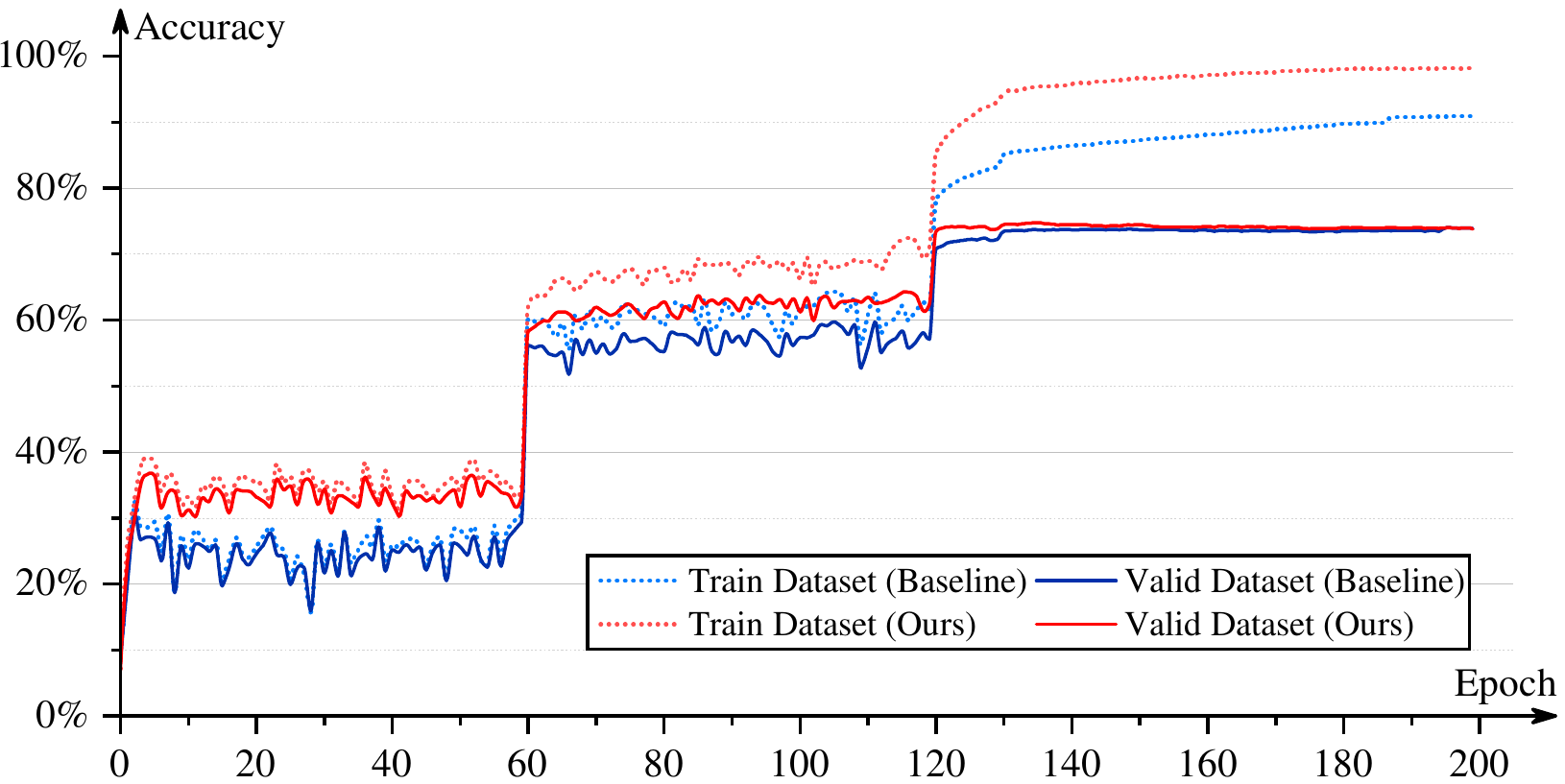}
    }
    \subfloat[ResNet-50 ImageNet top-5]
    {
        \includegraphics[width=0.42\linewidth]{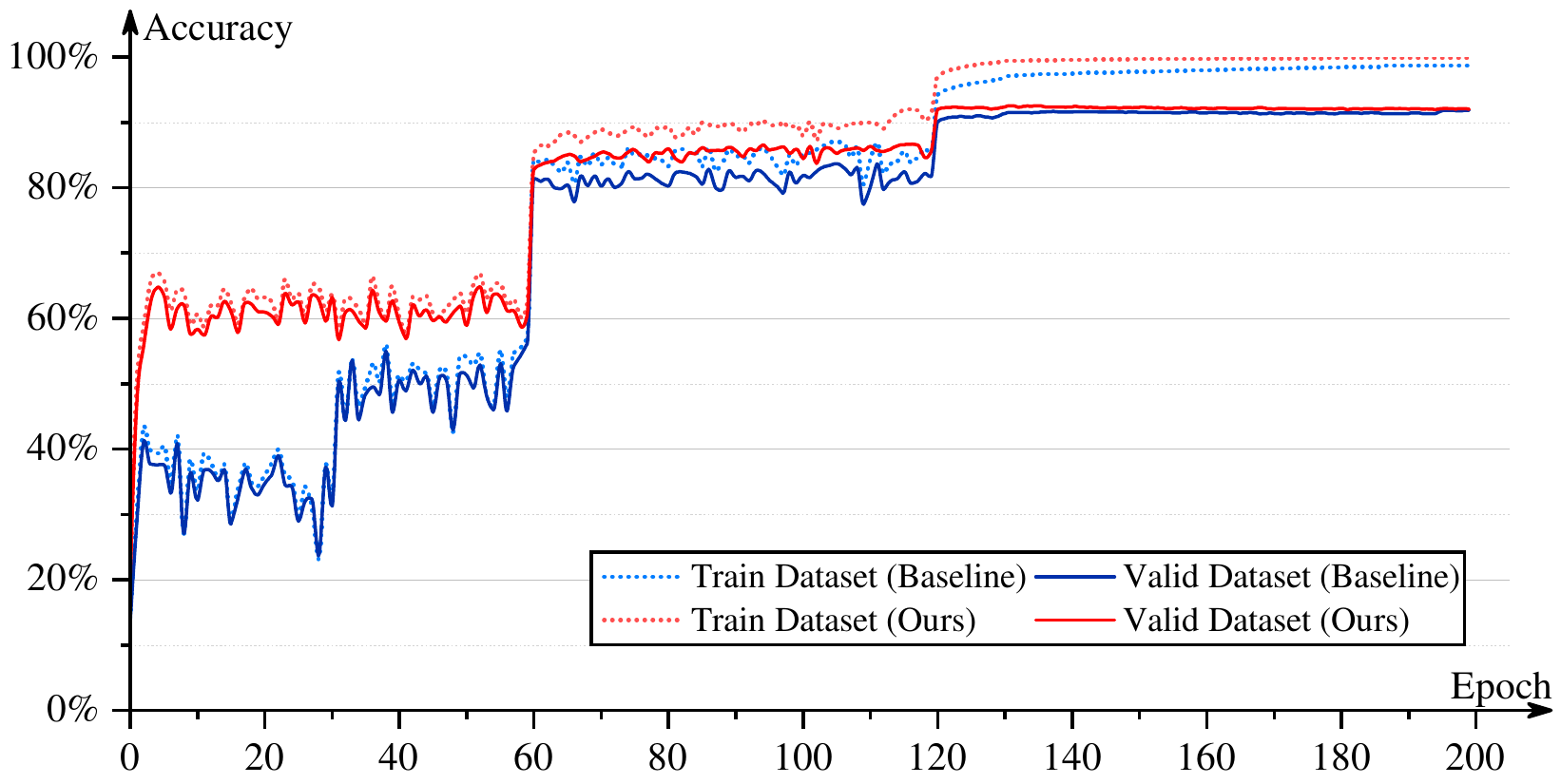}
    }

    \subfloat[ResNet-101 ImageNet top-1]
    {
        \includegraphics[width=0.42\linewidth]{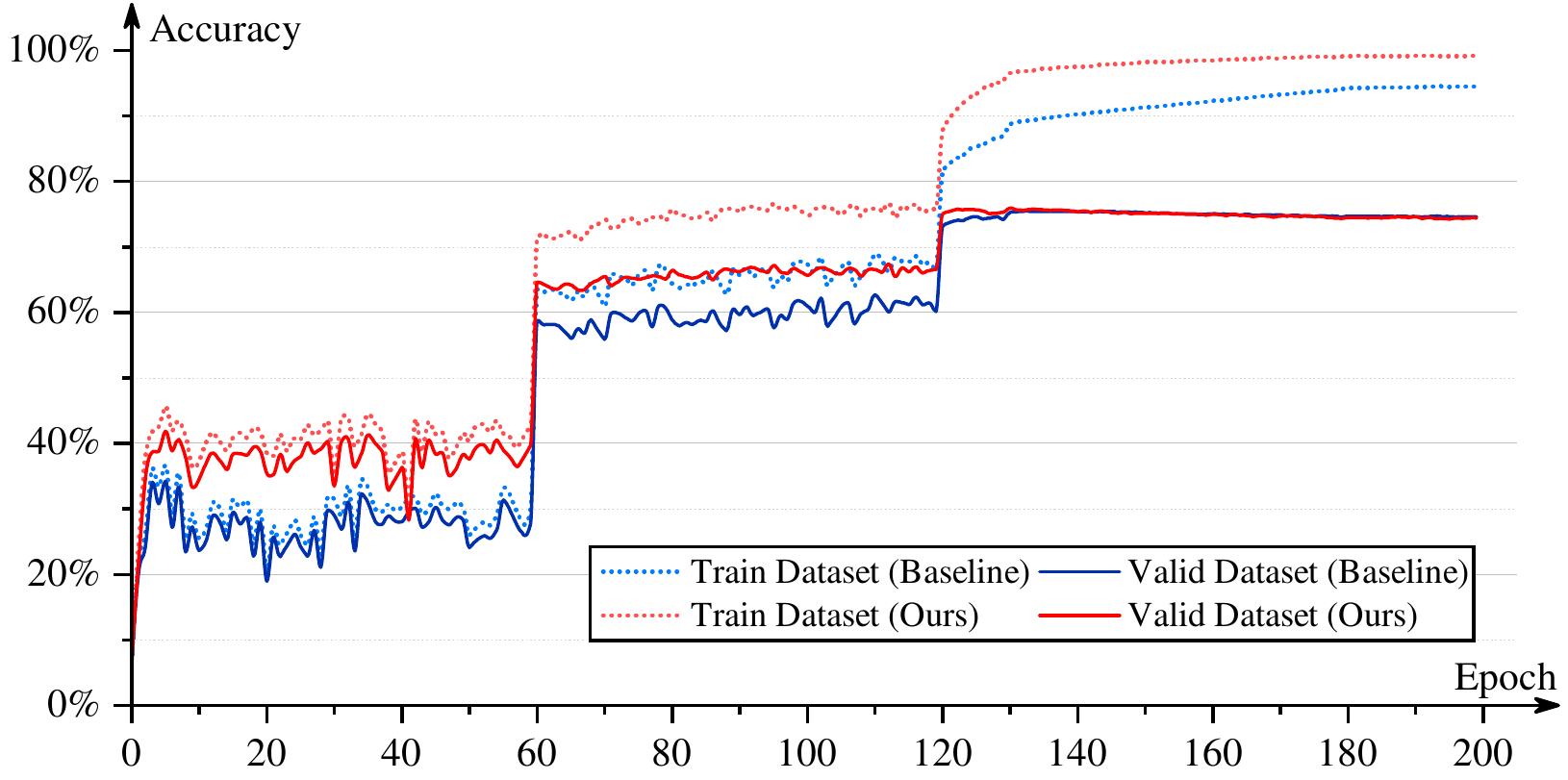}
    }
    \subfloat[ResNet-101 ImageNet top-5]
    {
        \includegraphics[width=0.42\linewidth]{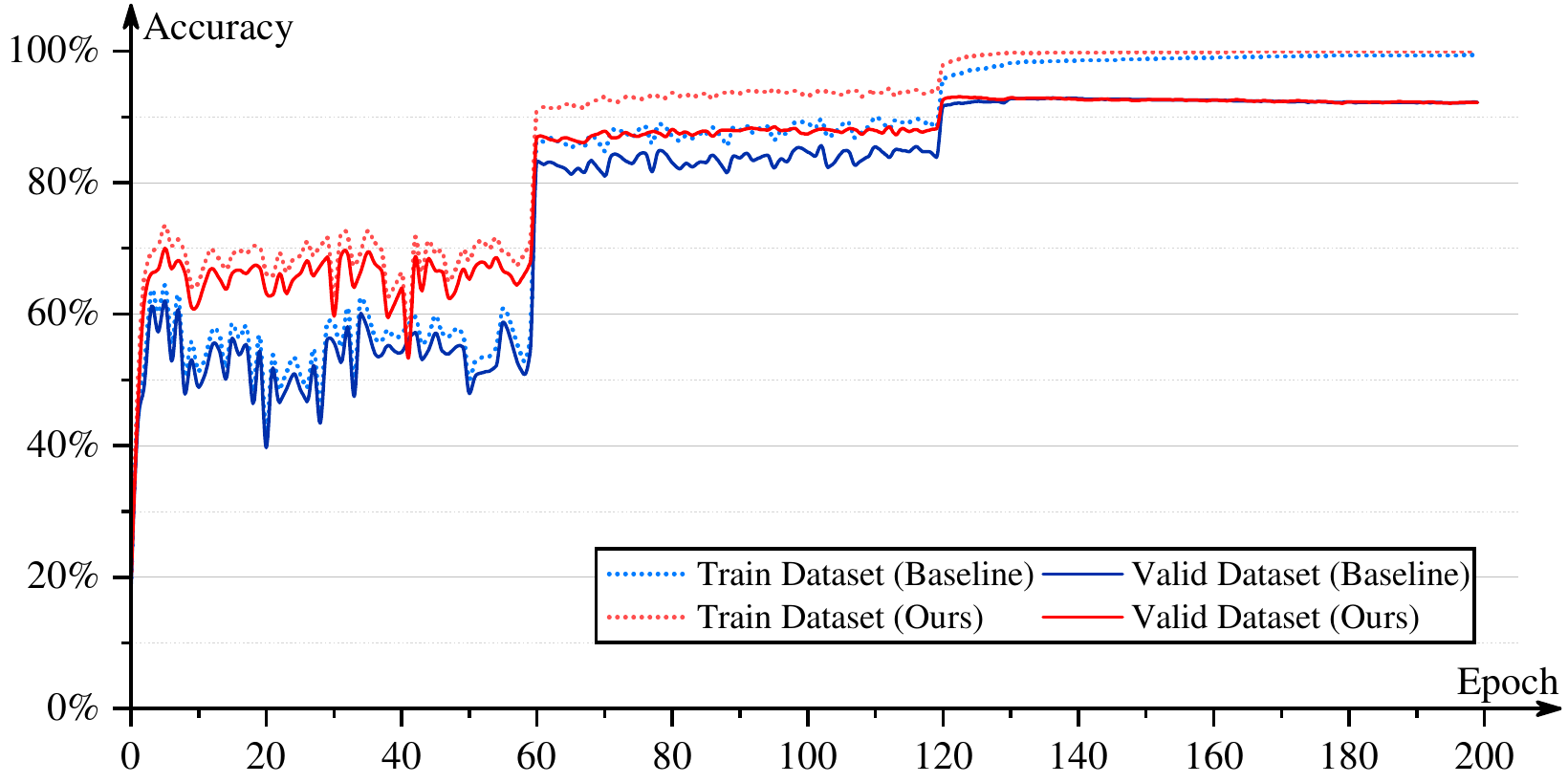}
    }
    \caption{Accuracy curves}
    \label{xunliantuxiang}
\end{figure}


\end{document}